\documentclass[a4paper,twoside,11pt]{article}
\usepackage[accepted]{aistats2018}
\usepackage[utf8]{inputenc}
\usepackage{amsfonts,amsmath,amsthm,amssymb}       
\usepackage[usenames,dvipsnames]{xcolor}
\usepackage{graphicx}
\usepackage{booktabs}
\usepackage{floatrow}
\usepackage{hyperref}       
\usepackage{url}            
\usepackage{booktabs}       
\usepackage{amsfonts}       
\usepackage{fullpage}
\usepackage{stmaryrd}

\usepackage{graphicx}
 \usepackage{subcaption}
\usepackage{amsfonts,amssymb,amsmath,amsthm}
\usepackage{color}
\usepackage{url}
\usepackage{tablefootnote}
\usepackage{mathrsfs}
\usepackage{enumitem}
\usepackage{array}
\usepackage{algorithm,algpseudocode}
\usepackage{wrapfig}
\usepackage{xcolor}
\usepackage{mathtools}

\newlength\myindent
\newcommand{\argmax}{\operatornamewithlimits{argmax}}
\newcommand{\argmin}{\operatornamewithlimits{argmin}}

\newtheorem{proposition}{Proposition}
\newtheorem{lemma}{Lemma}
\newtheorem{theorem}{Theorem}

\newtheorem{definition}{Definition}

\newcommand{\bb}[1]{\mathbb{#1}}

\newcommand{\vb}{\mathbf{v}}

\newcommand{\bs}{\boldsymbol}

\def\x{\boldsymbol{x}}

\def\z{\boldsymbol{z}}

\def\g{\boldsymbol{g}}
\def\r{\boldsymbol{r}}
\def\w{\boldsymbol{w}}
\def\vb{\boldsymbol{v}}

\def\u{\boldsymbol{u}}

\def\calB{\mathcal{B}}

\global\long\def\tr{\operatorname{Tr}}
\global\long\def\sign{\operatorname{sign}}
\global\long\def\var{\operatorname{Var}}

\makeatletter
\newtheorem*{rep@theorem}{\rep@title}
\newcommand{\newreptheorem}[2]{%
\newenvironment{rep#1}[1]{%
 \def\rep@title{#2 \ref{##1}}%
 \begin{rep@theorem}}%
 {\end{rep@theorem}}}
\makeatother

\newreptheorem{theorem}{Theorem}
\newreptheorem{lemma}{Lemma}
\newreptheorem{corollary}{Corollary}

\DeclarePairedDelimiterX{\infdivx}[2]{}{}{%
  #1\;\delimsize\|\;#2%
}

\begin{document}

%

%

\renewcommand{\thefootnote}{\fnsymbol{footnote}}

\aistatstitle{Robustness of classifiers to uniform $\ell_p$ and Gaussian noise}



\aistatsauthor{ Jean-Yves Franceschi \And Alhussein Fawzi \And  Omar Fawzi}
\aistatsaddress{ Ecole Normale Supérieure de Lyon\\LIP, UMR 5668 \And  UCLA Vision Lab\footnotemark[1] \And Ecole Normale Supérieure de Lyon\\LIP, UMR 5668 }

\begin{abstract}

We study the robustness of classifiers to various kinds of random noise models. In particular, we consider noise drawn uniformly from the $\ell_p$ ball for $p \in [1,\infty]$
and Gaussian noise with an arbitrary covariance matrix. We characterize this robustness to random noise in terms of the distance to the decision boundary of the classifier. This analysis applies to linear classifiers as well as classifiers with locally approximately flat decision boundaries, a condition which is satisfied by state-of-the-art deep neural networks. The predicted robustness is verified experimentally.

\end{abstract}

\section{Introduction}
\renewcommand{\thefootnote}{\fnsymbol{footnote}}
\footnotetext[1]{Now at DeepMind.}
\renewcommand{\thefootnote}{\arabic{footnote}}

Image classification techniques have recently witnessed major advances leading to record performances on challenging datasets \cite{he2015deep, krizhevsky2012imagenet}. Besides reaching low classification error, it is equally important that classifiers deployed in real-world environments correctly classify perturbed and noisy samples. Specifically, when a sufficiently small perturbation alters a sample, it is desirable that the estimated label of the classifier remains unchanged. Altering perturbations can take various forms, such as additive perturbations, geometric transformations or occlusions for image data. The analysis of the robustness of classifiers under these perturbation regimes is crucial for unraveling their fundamental vulnerabilities. For example, state-of-the-art image classifiers have recently been empirically shown to be vulnerable to well-sought imperceptible additive perturbations \cite{biggio2013evasion, szegedy2013intriguing}, and to more physically plausible nuisances in \cite{sharif2016accessorize}. The goal of this paper is to derive precise quantitative results on the robustness of general classifiers to \textit{random} noise. 


We specifically analyze two random noise models.
Under our first perturbation model, we assume that noise is sampled uniformly at random from the $\ell_p$ ball, for $p \in [1, \infty]$. Different values of $p$ allow us to model very different noise regimes; e.g., $p=1$ corresponds to sparse noise, whereas $p=\infty$ models dense noise typically resulting from signal quantization.
Under our second perturbation regime, the noise is modeled as Gaussian with arbitrary covariance matrix $\Sigma$. Our contributions are summarized as follows:

\begin{itemize}
	\item For linear classifiers, we characterize up to constants the robustness to random noise, as a function of the distance to the decision boundary. We show in particular that, provided the weight vector of the linear classifier is randomly chosen, the robustness to random noise (uniform and Gaussian) scales as $\sqrt{d}$ times the distance to the decision boundary.
	\item We extend the results to  nonlinear classifiers, and show that when the decision boundary is locally approximately flat (which is the case for state-of-the-art classifiers), the above result notably holds.
	\item Through experimental evidence on state-of-the-art image classifiers (deep nets), we show that the proposed bounds predict accurately the robustness of such classifiers. We finally show that our analysis predicts the high robustness of such classifiers to image quantization, which confirms previous empirical evidence. 
\end{itemize}


\textbf{Related work.} The robustness properties of linear and kernel SVM classifiers have been studied in \cite{xu2009robustness, biggio2013evasion}, and robust optimization approaches for constructing robust classifiers have been proposed \cite{sra2012optimization, lanckriet2003robust}. More recently, the robustness properties of deep neural networks have been investigated. In particular, \cite{szegedy2013intriguing} shows that deep neural networks are not robust to worst-case, or adversarial, perturbations. Several works have followed and attempted to provide  explanations to the vulnerability \cite{goodfellow2014, tabacof2015exploring, tanay2016boundary, sabour2016adversarial}. In particular, it was shown theoretically that the ratio of robustness to random noise and robustness to adversarial perturbations measured in the $\ell_2$ norm scales as $\sqrt{d}$ for linear classifiers in \cite{fawzi2015a} and more general classification functions in \cite{nips2016_ours}. Therefore, when the data is sufficiently high dimensional, the robustness to adversarial perturbations is very small, which gives an explanation to the imperceptible nature of such perturbations. 
Our work generalizes \cite{nips2016_ours} to broader noise regimes, such as sparse noise, quantization noise, or correlated Gaussian noise. Indeed, we follow a similar methodology to that of  \cite{nips2016_ours}, where we first establish results for the linear case, and then extend the results to nonlinear classifiers satisfying a locally approximately flat decision boundary.

\textbf{Outline.} This paper is organized as follows. Section~\ref{sec:definitions} introduces the framework of the robustness to random and adversarial perturbations. Section~\ref{sec:linear} presents theoretical estimates of such robustnesses for linear classifiers, which are generalized in Section~\ref{sec:robustness_nonlinear} for classifiers with a locally approximately flat decision boundary. Section~\ref{sec:experiments} then details experiments showing the validity of our bounds for state-of-the-art classifiers and exposing some applications of our results.


\section{Definitions and notations}
\label{sec:definitions}


Let $f: \bb{R}^d \rightarrow \bb{R}^L$ be a $L$-class classifier. The estimated label of a datapoint $\x \in \bb{R}^d$ is set to $g(\x) = \argmax_{k} f_k(\x)$, where $f_k(\x)$ denotes the $k$th component of $f(\x)$.  Our goal in this paper is to analyze the robustness of $f$ to random perturbations of the input. For that, we consider an arbitrary distribution $\nu$ on $\bb{R}^d$ that we interpret as giving the direction $\vb$ of the noise, and we measure the length of the minimal scaling applied to $\vb$ required to change the estimated label of $f$ at $\x$ with probability at least $\varepsilon$. More precisely, let $\vb$ be a random variable distributed according to $\nu$; for a given $\varepsilon>0$, we define $r_{\nu,\varepsilon}(\x)$ as:
\begin{align}
\label{eq:robustness_random}
r_{\nu,\varepsilon}(\x) = \min_{\alpha} \left\{ \left|\alpha\right| \text{ s.t. } \bb{P}\left\{g(\x+\alpha \vb) \neq g(\x) \right\} \geq \varepsilon\right\}. 
\end{align}
If the set is empty, we set $r_{\nu,\varepsilon}(\x) = + \infty$.\footnote{We should also technically consider the closure of the set to ensure the minimum is achieved, but we will avoid such technicalities throughout the paper as they are of no relevance for our study.} In this paper, we will focus on two families of choices for $\nu$.

The first family is parameterized by a real number $p \in [1,\infty]$. The distribution $\nu$ is then the uniform distribution over the unit ball of $\ell_p^d$, i.e., $\calB_p = \{\x \in \bb{R}^d: \| \x \|_p \leq 1 \}$ where $\| \x \|_{p} = (\sum_{i=1}^d x_i^p)^{1/p}$. For this setting of distribution $\nu$, we write $r_{p,\varepsilon}( \x ) = r_{\nu,\varepsilon}( \x)$. Observe that the Euclidean norm $\| . \|_{2}$ is invariant under an orthonormal basis change, but this is not the case for $\| . \|_{p}$ when $p \neq 2$, i.e., it depends on the basis that is chosen to write the signal $\x$; hence, this dependence also holds for $r_{p,\varepsilon}( \x )$. Different choices of $p$ allow us to span a range of realistic noise models. For example, choosing $p=1$ leads to sparse noise vectors modeling salt and pepper noise, while $p=\infty$ leads to uniform noise vectors that allow us to model noise resulting from signal quantization \cite[Chapter~4.5]{bovik2010handbook}. An illustration of the different noise regimes can be found in Figure~\ref{fig:illustration_noises}.

The second family is parameterized by an arbitrary positive definite matrix $\Sigma \in \bb{R}^{d \times d}$, that will generally be normalized with $\tr(\Sigma) = 1$ to fix the scale. The distribution $\nu$ is then the multivariate normal distribution with mean $\bf{0}$ and covariance matrix $\Sigma$. We use the notation $r_{\Sigma,\varepsilon}(\x) = r_{\nu,\varepsilon}( \x)$ for this setting. A special case of this family is therefore the additive white Gaussian noise (where $\Sigma = \frac{I}{d}$); note however that this family is much broader and can model a noise that is correlated with the input $\x$, as no assumption is made on $\Sigma$.

In the remainder of this paper, our goal is to derive bounds on the robustness of classifiers $f$ to random noise sampled from either of these two families. To do so, we first define a key quantity for our analysis, the robustness to worst-case perturbations:
\begin{align}
\label{eq:robustness_adversarial}
\r_p^*(\x) = \argmin_{\r} \left\{\| \r \|_p \text{ s.t. } g(\x+\r) \neq g(\x)\right\}.
\end{align}
In other words, $\| \r_p^*(\x) \|_p$ quantifies the length of the minimal perturbation required to change the estimated label of the classifier, or equivalently, the distance from the data point $\x$ to the decision boundary of the classifier. $\r_p^* (\x)$ is often alternatively referred to as an \textit{adversarial} perturbation, as it corresponds to the least noticeable perturbation an adversary would apply to fool a classifier. Note that, like $r_{p,\varepsilon}( \x )$, it heavily depends on the choice of norm $\ell_p$, and thus on the choice of orthonormal basis. Figure~\ref{fig:illustration_adversarial_perturbations} illustrates the dependence on $p$ of this perturbation. Such perturbations, which have been the subject of intense studies, will be used to derive guarantees on the robustness to random noise.

\begin{figure}
	\includegraphics[width=0.6\textwidth]{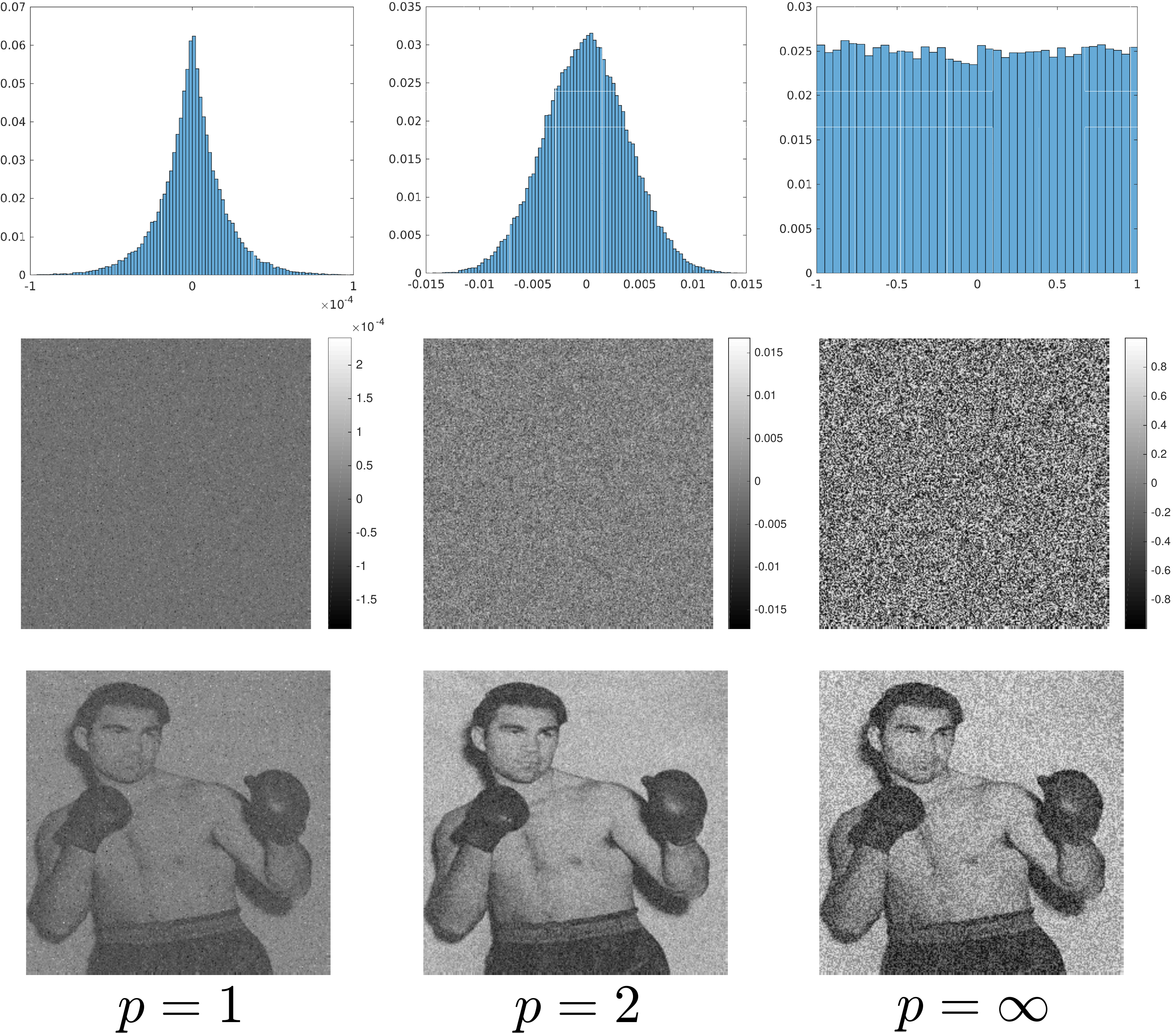}
	\caption{\label{fig:illustration_noises}Illustration of noise with different values of $p$. First row: histogram of uniformly sampled noise from the unit ball of $\ell_p^d$. Second row: Example of noise image. Third row: Example of noisy image. Note that different values of $p$ result in perceptually different noise images.}
\end{figure}

\begin{figure}
	\includegraphics[width=0.6\textwidth]{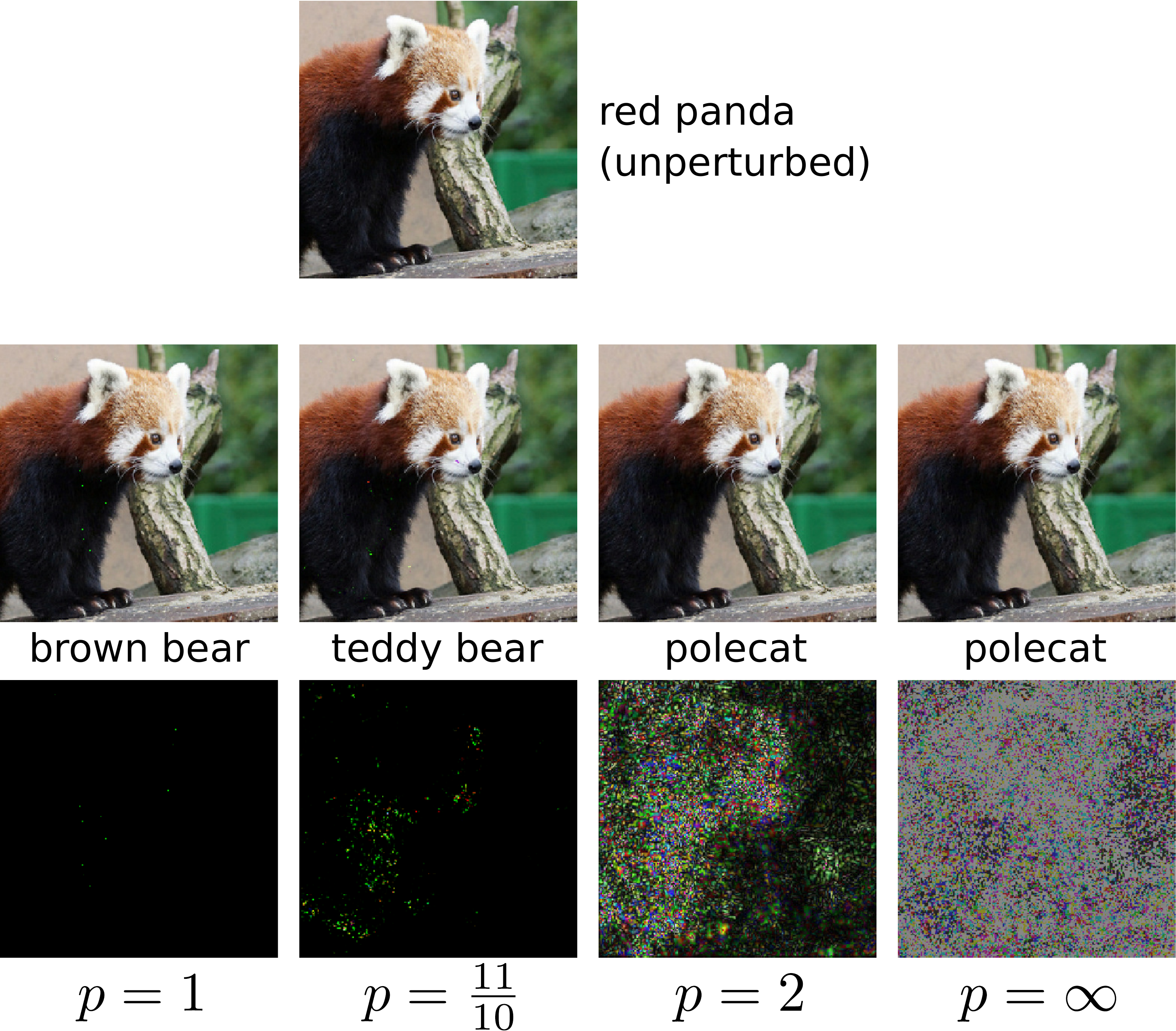}
	\caption{\label{fig:illustration_adversarial_perturbations}Illustration of adversarial perturbations with different values of $p$. First row: original image and its classification. Second row, for each column, from bottom to top: chosen $p$, adversarial perturbation, perturbed image with its classification. When $p \to \infty$, the perturbation tends to be distributed over all pixels; when $p \to 1$, it tends to be distributed over few pixels. Adversarial perturbations were estimated on the VGG-19 classifier \cite{simonyan2014very} using the method presented in \cite{moosavi2015deepfool}.}
\end{figure}





In the next sections, we characterize the robustness of linear and nonlinear classifiers to random perturbations in terms of the robustness to worst-case perturbations. In the case of Gaussian random noise, we focus for $\r_p^*\left(\x\right)$ on the norm $\| . \| = \| . \|_2$, even though all our results can be generalized to $p$-norms.

\section{Robustness of linear classifiers}
\label{sec:linear}

For simplicity of exposition, we state our results for binary classifiers, and we extend the results for multi-class classifiers in the supplementary material. The proofs may also be found in the supplementary material.

We consider in this section the particular case where $f$ is a \emph{linear} classifier, i.e., all the $f_k$'s are linear functions. In particular, in the binary case, the setting can be simplified by considering a single linear function
\begin{equation}
	f : \x \mapsto \bs{w}^T \x + \bs{b}.
\end{equation}
In this case,\footnote{In the general multi-class setting, this corresponds to $f_0 = f$ and $f_1 = 0$.} $g(\x) = 1$ if and only if $f(\x) > 0$.

\subsection{Uniform $\ell_p$ noise}

The following result bounds the robustness of a linear classifier to uniformly random noise with respect to its robustness to adversarial perturbations, for any norm $\ell_p$.

\begin{theorem}
\label{thm:lp_noise}
Let $p \in [1, \infty]$. Let $p' \in [1,\infty]$ be such that $\frac{1}{p} + \frac{1}{p'} = 1$.
There exist constants $\varepsilon_0, \zeta_1(\varepsilon), \zeta_2(\varepsilon)$ such that, for all $\varepsilon < \varepsilon_0$:
\begin{equation*}
\zeta_1(\varepsilon) d^{1/p} \frac{ \| \w \|_{p'} }{\| \w \|_{2} } \leq \frac{r_{p,\varepsilon}(\x)}{\| \r^*_{p}(\x) \|_p} \leq \zeta_2(\varepsilon) d^{1/p} \frac{ \| \w \|_{p'} }{\| \w \|_{2} }.
\end{equation*}
We can take $\zeta_1(\varepsilon) = C\sqrt{\varepsilon}$,\footnote{We show in the supplementary material that for $p > 1$, we can also choose $\zeta_1(\varepsilon) = \frac{C'}{\sqrt{\ln \frac{1}{\varepsilon}}}$ for some constant $C'$.} and $\zeta_2(\varepsilon) = \sqrt{\frac{1}{c-\sqrt{c'\varepsilon}}}$, for some constants $C, c, c'$.
\end{theorem}
More details on constants $C, c, c'$ are available in the supplementary material.
In words, our result demonstrates that $r_{p,\varepsilon}(\x)$ is well estimated by $\| \r^*_{p}(\x) \|_p$ times a multiplicative factor that is independent of $\x$ and is of the order $d^{1/p} \frac{ \| \w \|_{p'} }{\| \w \|_{2} }$. The special case $p = 2$, for which this multiplicative factor becomes $\sqrt{d}$, was previously shown in~\cite{nips2016_ours} and \cite{fawzi2015a}. For $p \neq 2$, this factor depends on the choice of the classifier through vector $\w$. Such a dependence was to be expected as the $p$-norm for $p \neq 2$ depends on the choice of basis. This dependence takes into account the relation between this choice of basis to write the signal and the direction $\w$ chosen by the classifier. For example, when $\w = (1 \ 0 \ \dots \ 0)^T$, we have a classifier that only uses the first component of the signal. So for $p = \infty$, the problem effectively becomes one-dimensional as only the first coordinate matters and we have $r_{\infty}(\x) = \| \r^*_{\infty}(\x) \|_{\infty}$.

%
%
Nevertheless, for a typical choice of the vector $\w$ of the linear classifier, this factor stays of order $\sqrt{d}$ if $p > 1$.

\begin{proposition}
\label{prop:MeanBoundsRandom}
For any $p\in\left(1,\infty\right]$, if $\w$ is a random direction uniformly distributed over the unit
$\ell_{2}$-sphere, then, as $d \to \infty$,
\begin{equation*}
\frac{d^{1/p} \frac{\left\Vert \w\right\Vert _{p'}}{\| \w \|_2}}{\sqrt{d}} \xrightarrow[\text{a.s.}]{}\sqrt{2}\left(\frac{\Gamma\left(\frac{2p-1}{2\left(p-1\right)}\right)}{\sqrt{\pi}}\right)^{1-\frac{1}{p}}.
\end{equation*}
Moreover, for $p = 1$,
\begin{equation*}
\frac{d \frac{\left\Vert \w \right\Vert _{\infty}}{\| \w \|_2}}{\sqrt{2d \ln d}} \xrightarrow[\text{a.s.}]{} 1 \text{.}
\end{equation*}
\end{proposition}

While this result is only asymptotic and valid for random decision hyperplanes, we experimentally show in Section~\ref{sec:experiments} that its dependence in $p$ allows us to propose an estimate providing a very good approximation of the robustness to random noise.

\subsection{Gaussian noise}

In the case where the uniformly random noise is replaced by a Gaussian noise with a given covariance matrix $\Sigma$, we can similarly characterize the ratio $\frac{r_{\Sigma,\varepsilon}(\x)}{\| \r^*_{2}(\x) \|_2}$ as a function of $\Sigma$ and $\w$ as follows.
\begin{theorem}
\label{thm:gaussian_noise}
Let $\Sigma$ be a $d \times d$ positive semidefinite matrix with $\tr(\Sigma) = 1$.\footnote{Note that the condition $\tr(\Sigma) = 1$ is not needed for the statement but its motivation is to fix the scale of $r_{\Sigma,\varepsilon}(\x)$.} There exist constants $\varepsilon_0', \zeta_1'(\varepsilon), \zeta_2'(\varepsilon)$ such that, for all $\varepsilon < \varepsilon_0'$:
\begin{equation*}
\zeta_1'(\varepsilon) \frac{ \| \w \|_{2} }{\| \sqrt{\Sigma} \w \|_{2} } \leq \frac{r_{\Sigma,\varepsilon}(\x)}{\| \r^*_{2}(\x) \|_2} \leq \zeta_2'(\varepsilon)  \frac{ \| \w \|_{2} }{\| \sqrt{\Sigma} \w \|_{2} }.
\end{equation*}
We can take $\varepsilon_0' = \frac{1}{3}$, $\zeta_1'(\varepsilon) = \sqrt{\frac{1}{2\ln\left(\frac{1}{\varepsilon}\right)}}$ and $\zeta_2'(\varepsilon) = \sqrt{\frac{1}{1-\sqrt{3\varepsilon}}}$.
\end{theorem}

In this case, the multiplicative factor between robustnesses to random and adversarial perturbations is of the order $\frac{ \| \w \|_{2} }{\| \sqrt{\Sigma} \w \|_{2} }$. Note that this factor lies in between $\lambda_{\max}(\sqrt{\Sigma})^{-1}$ and $\lambda_{\min}(\sqrt{\Sigma})^{-1}$. However, these values correspond to extremal cases, and for most choices of $\w$, this factor will be determined by a convex combination of eigenvalues of $\Sigma$. More precisely, if $\u_1, \dots, \u_d$ are the eigenvectors of $\Sigma$ with eigenvalues $\lambda^2_1, \dots, \lambda^2_d$ and assuming $\| \w \|_2 = 1$ (without loss of generality), the factor is given by a weighted  average of the eigenvalues $\lambda_1^2, \dots, \lambda_d^2$:
\begin{align*}
\frac{1}{ \sqrt{\sum_{i=1}^{d} \lambda_i^2 |\u_i^T \w|^2 }}.
\end{align*}
In particular, if $\Sigma = \frac{1}{d}I_d$, then the factor is $\sqrt{d}$. Even more generally, for typical choices of $\w$ we expect $|\u_i^T \w|^2$ be of order $\frac{1}{d}$, in which case the factor will also be of order $\sqrt{d}$.



%

\section{Robustness of nonlinear classifiers}
\label{sec:robustness_nonlinear}

We now consider the general case where $f$ is a nonlinear classifier. The goal of this section is to derive relations between $r_{p,\varepsilon}(\x)$ and $\| \r_p^*(\x) \|_p$ in this general case, under a reasonable hypothesis on the geometry of the decision boundary.

\subsection{Locally Approximately Flat (LAF) Decision Boundary Model}

Before giving the formal definition, let us describe the main idea behind the \textit{Locally Approximately Flat} (LAF) decision boundary model. This model requires that the decision boundary can be locally sandwiched between two hyperplanes that are parallel to the tangent hyperplane. We do not ask for this to hold for every point on the decision boundary, but only to hold for the closest points on the decision boundary of our data points. 

\begin{definition}[LAF model]
	Let $f$ be a binary classifier with smooth\footnote{This is a strong assumption that is only used here for the sake of exposition. Actually, the tangent $\mathcal{T}(\x^*)$ in the definition can be replaced by any hyperplane intersecting the decision boundary at $\x^*$. Then, in the results conditioned by the LAF model, the gradient of $f$ at $\x^*$ can be replaced by a normal vector to this plane.} decision boundary $\mathcal{S} = \{x \in \bb{R}^d : f(\x) = 0\}$. For $\x^* \in \mathcal{S}$, define $\mathcal{T}(\x^*)$ to be the hyperplane tangent to $\mathcal{S}$ at point $\x^*$. For $\x \in \bb{R}^d$ and $\x^* \in \mathcal{S}$, we define $\mathcal{H}_{\gamma}^{-}(\x,\x^*)$ to be the halfspace of points that are on the side of $\x$ of the hyperplane parallel to $\mathcal{T}(\x^*)$ that passes though the point $\gamma \x + (1-\gamma)\x^*$. Similarly, $\mathcal{H}_{\gamma}^{+}(\x,\x^*)$ is the halfspace of points that are not on the side of $\x$ of the hyperplane parallel to $\mathcal{T}(\x^*)$ that passes though the point $\x^* + \gamma (\x^* - \x)$ (see Figure~\ref{fig:LAF}).

	We say that $f$ is $(\gamma, \eta)$-Locally Approximately Flat at point $\x$ if for $\x^* \in \mathcal{S}$ minimizing $\| \x - \x^*\|_p$, the set $\mathcal{H}_{\gamma}^{-}(\x,\x^*) \cap \calB_p(\x, \eta)$ is classified as $\x$ and $\mathcal{H}_{\gamma}^+(\x,\x^*) \cap \calB_p(\x, \eta)$ is classified differently from $\x$. Here $\calB_{p}(\x, \eta)$ is the $\ell_p$-ball centered at $\x$ with radius $\eta$.

\begin{figure}
	\centering
	\includegraphics[width=0.6\textwidth]{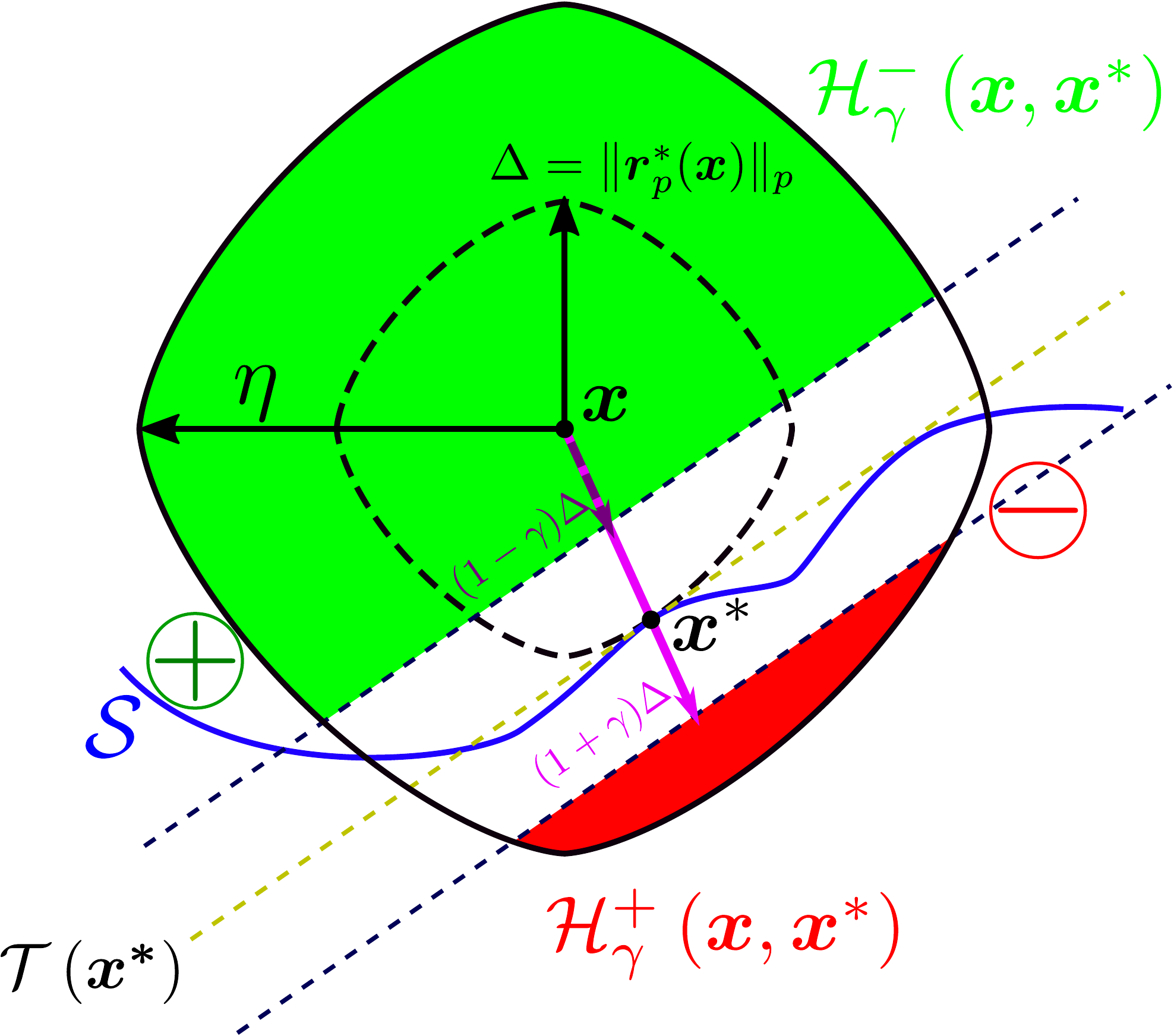}
	\caption{\label{fig:LAF}Illustration of a $(\gamma, \eta)$-LAF classifier with $p = \frac{3}{2}$. $\mathcal{S}$ is the decision boundary of $f$, separating instances of the same predicted class as $\x$ (on the side of $\x$) from instances whose classification differs from $\x$ (on the other side). By definition, all inputs in $\mathcal{H}_{\gamma}^+(\x,\x^*) \cap \calB_p(\x, \eta)$ (red area, other side of $\mathcal{S}$) are classified differently from $\x$, while the inputs in $\mathcal{H}_{\gamma}^-(\x,\x^*) \cap \calB_p(\x, \eta)$ (green area, same side of $\mathcal{S}$ as $\x$) are classified as $\x$.}
\end{figure}

\end{definition}

The LAF model assumes that the decision boundary can be \textit{locally} approximated by a hyperplane, in the vicinity of images $\x$ sampled from the data distribution. It should be noted that, in order to be able to define locality, we need a distance measure and thus the LAF property depends implicitly on the choice of norm. For $\gamma = 0$, the LAF model corresponds to locally exactly flat decision boundaries (no curvature). If in addition $\eta = \infty$, this corresponds to a linear decision boundary.

Prior empirical evidence has shown that state-of-the-art deep neural network classifiers have decision boundaries that are approximately flat along random directions \cite{warde2016adversarial, nips2016_ours}. Normal two-dimensional cross-sections (along random directions) of the decision boundary are illustrated in Figure~\ref{fig:two_dimensional_cross_sections}. Note that such cross-sections have very low curvature, thereby providing evidence that the LAF assumption holds approximately (at least, with high probability) for complex classifiers, such as modern deep neural networks. It should further be noted that the LAF model is tightly related to the curvature condition of the decision boundary in \cite{nips2016_ours}. The LAF model, however, does not assume any regularity condition on the decision surface, which is nonsmooth in many settings (e.g., deep neural networks due to piecewise linear activation functions). Finally, it should be noted that, for the sake of clarity, we assumed that the entire set $\mathcal{H}_{\gamma}^{+}(\x,\x^*) \cap \calB_p(\x, \eta)$ (respectively, $\mathcal{H}_{\gamma}^-(\x,\x^*) \cap \calB_p(\x, \eta)$) is classified differently from $\x$ (respectively, similarly to $\x$); however, the results in this section hold even if these conditions are only satisfied with high probability.


\begin{figure}
	\centering
	\includegraphics[width=0.85\textwidth]{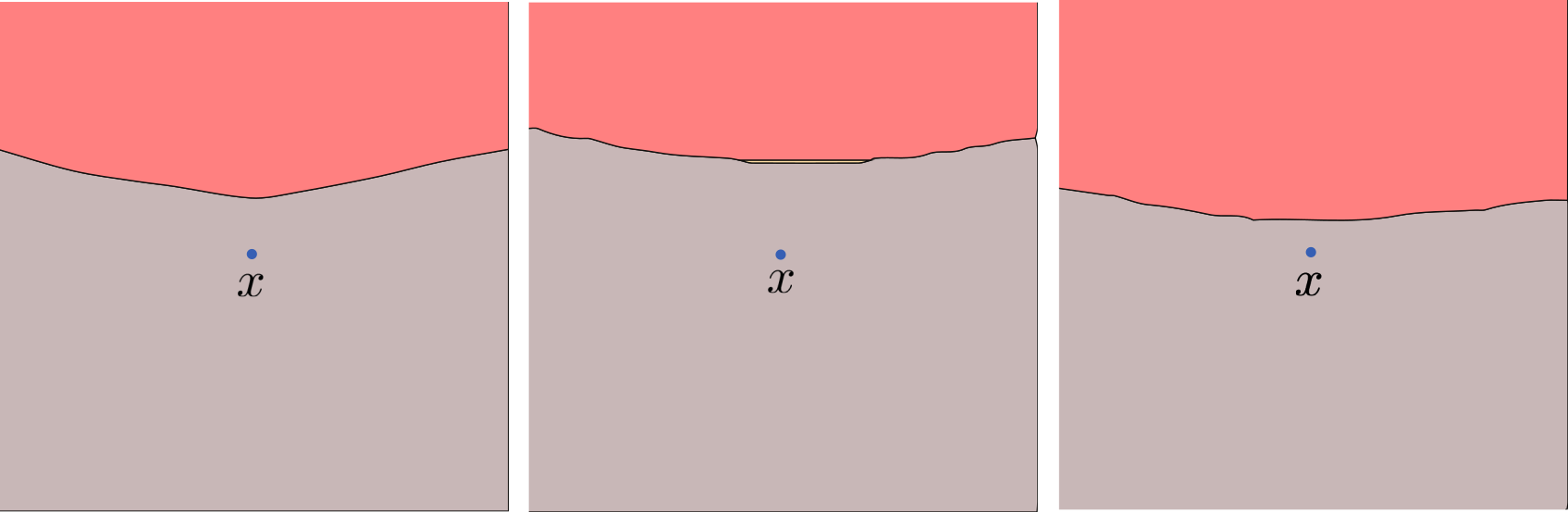}
	\caption{\label{fig:two_dimensional_cross_sections}Two dimensional normal cross-sections of the decision boundary of a deep network classifier along random directions, in the vicinity of different natural images (denoted by $\x$ for each cross-section). The CaffeNet architecture \cite{jia2014} trained on ImageNet \cite{russakovsky2015imagenet} was used.}
\end{figure}




\subsection{Robustness Results Under LAF Model}

Our next result shows that, provided $f$ is Locally Approximately Flat, a very similar result to Theorem~\ref{thm:lp_noise} holds, with the normal vector $\w$ replaced by the gradient of $f$ at the point $\x^*$ of the boundary that is closest to $\x$. It should be noted that for nonlinear classifiers, the gradient $\nabla f(\x^*)$ plays the same role as $\w$ for linear classifiers, as it is normal to the tangent to the decision boundary at $\x^*$.

\begin{theorem}
\label{thm:lp_noise_laf}
Let $p \in [1, \infty]$. Let $p' \in [1,\infty]$ be such that $\frac{1}{p} + \frac{1}{p'} = 1$. Let $\varepsilon_0, \zeta_1(\varepsilon), \zeta_2(\varepsilon)$ be as in Theorem~\ref{thm:lp_noise}. Then, for all $\varepsilon < \varepsilon_0$, the following holds.

Assume $f$ is a classifier that is $(\gamma, \eta)$-LAF at point $\x$ and $\x^*$ be such that $\r^*_{p}(\x) = \x^* - \x$. Then:
\begin{equation*}
(1-\gamma) \zeta_1(\varepsilon) d^{1/p} \frac{ \| \nabla f(\x^*) \|_{p'} }{\| \nabla f(\x^*) \|_{2} } \leq  \frac{r_{p,\varepsilon}(\x)}{\| \r^*_{p}(\x) \|_p} \leq (1+\gamma) \zeta_2(\varepsilon) d^{1/p} \frac{ \| \nabla f (\x^*) \|_{p'} }{\| \nabla f(\x^*) \|_{2} } ,
\end{equation*}
provided $$\eta \geq (1+\gamma) \zeta_2(\varepsilon) d^{1/p} \frac{ \| \nabla f (\x^*) \|_{p'} }{\| \nabla f(\x^*) \|_{2}} \left\|\r_p^*\left(\x\right) \right\|_p .$$
\end{theorem}

In the case where $\nabla f (\x^*)$ is uncorrelated with the basis used to write the signal (which we model by taking for $\nabla f (\x^*)$ a random direction in the $\ell_2$ sphere), we obtain the same result as in Proposition~\ref{prop:MeanBoundsRandom} (i.e., by replacing $\w$ with $\nabla f (\x^*)$). This provides bounds on the robustness to random noise that only depend on $\| \r^*_2 \|_2$ and $d$.
We show that these asymptotic bounds provide accurate estimates of the empirical robustness in Section~\ref{sec:experiments}.

The result on Gaussian noise also holds for LAF classifiers.
\begin{theorem}
\label{thm:gaussian_noise_laf}
Let $\Sigma$ be a $d \times d$ positive semidefinite matrix with $\tr(\Sigma) = 1$. Let $\varepsilon'_0, \zeta'_1(\varepsilon), \zeta'_2(\varepsilon)$ as in Theorem~\ref{thm:gaussian_noise}. Then, for all $\varepsilon < \frac{1}{2}\varepsilon'_0$, the following holds.

Assume $f$ is a classifier that is $(\gamma, \eta)$-LAF at point $\x$ and $\x^*$ be such that $\r^*_{2}(\x) = \x^* - \x$. Then:
\begin{equation*}
(1-\gamma) \zeta'_1\left(\frac{\varepsilon}{2}\right) \frac{ \| \nabla f(\x^*) \|_{2} }{\| \sqrt{\Sigma} \nabla f(\x^*) \|_{2} } \leq  \frac{r_{\Sigma,\varepsilon}(\x)}{\| \r^*_{2}(\x) \|_2} \leq (1+\gamma) \zeta'_2\left(\frac{3\varepsilon}{2}\right) \frac{ \| \nabla f (\x^*) \|_{2} }{\| \sqrt{\Sigma} \nabla f(\x^*) \|_{2} },
\end{equation*}
provided, using $\psi(\varepsilon)= 8\tr\left(\Sigma^2\right) \ln \frac{4}{\varepsilon}$,$$\eta \geq (1+\gamma) (1 + \psi(\varepsilon)) \zeta'_2\left(\frac{3\varepsilon}{2}\right) \frac{ \| \nabla f (\x^*) \|_{2} }{\| \sqrt{\Sigma} \nabla f(\x^*) \|_{2} }\left\|\r_2^*\left(\x\right) \right\|_2.$$
\end{theorem}




\section{Experiments}
\label{sec:experiments}


\textbf{Robustness of a binary linear classifier to uniform random noise.} We now assess empirically our bounds for the robustness to random noise. We first consider the 10-class MNIST digit classification task \cite{lecun1998mnist}, and train a binary linear classifier separating digits 0 to 4 from digits 5 to 9, which achieves a performance of $84\%$ in the test set. To assess our analytical results, we compare these to an empirical estimate of the robustness to uniform random noise for different values of $p$. It is based on the combination of the expressions found in Theorem~\ref{thm:lp_noise} and Proposition~\ref{prop:MeanBoundsRandom}; for a fixed $\varepsilon$, our estimate is:
\begin{equation}
	\label{eq:estimate}
	\mathscr{E}_{r_p}\left(\x\right) = \zeta_0 \sqrt{d} \left(\frac{\Gamma\left(\frac{2p-1}{2\left(p-1\right)}\right)}{\sqrt{\pi}}\right)^{1-\frac{1}{p}} \left\| \r^*_{p}(\x) \right\|_p,
\end{equation}
where $\zeta_0$ is a constant.
The empirical robustness of Eq.~(\ref{eq:robustness_random}) is specifically computed  through an exhaustive search of smallest radius of the $\ell_p$ ball leading to an $\varepsilon$ fraction of misclassified samples. Note moreover that for this linear classifier, the worst-case robustness $\| \r_p^* \|_p$ is given by the distance to the hyperplane, and can therefore be computed in closed form (see the supplementary material). Figure~\ref{fig:MNIST_linear} illustrates the empirical robustness, our theoretical bounds and our estimate (i.e., upper and lower bounds of Theorem~\ref{thm:lp_noise}, and estimate of Eq.~(\ref{eq:estimate})) with respect to $p$. In addition to providing accurate upper and lower bounds for all the range of tested $p$-norms, observe that our estimate provides a remarkably accurate approximation of the robustness to random noise, for all $p$. Our analytical results hence correctly predict the robustness behavior of this classifier through a wide variety of noise models, and can therefore be used to predict the robustness in these regimes.

\textbf{Robustness of a multi-class deep neural network to uniform random noise.} We now consider a more complex classification setting, where we evaluate the robustness of the VGG-19 deep neural network on the multi-class ImageNet dataset of natural images \cite{russakovsky2015imagenet}. Similarly to our experiment for the linear classifier, we compare the empirical value of the robustness for different values of $p$ to our theoretical bounds from Theorem~\ref{thm:lp_noise_laf} and our estimate from Eq.~(\ref{eq:estimate}). Note that unlike the previous case, the worst-case robustness $\| \r_p^* \|_p$ cannot be obtained in closed-form for deep networks; we therefore estimate it using the algorithm described in \cite{moosavi2015deepfool}. The results are shown in Figure~\ref{fig:ImageNet_VGG}.
Observe that, once again, our estimate predicts accurately the robustness of the deep neural network for different values of $p$. Hence, despite the high nonlinearity of the deep network as a function of the inputs, our bounds established under the LAF assumption hold accurately for all tested values of $p$. 



\begin{figure}
	\centering
	\includegraphics[width=0.8\textwidth]{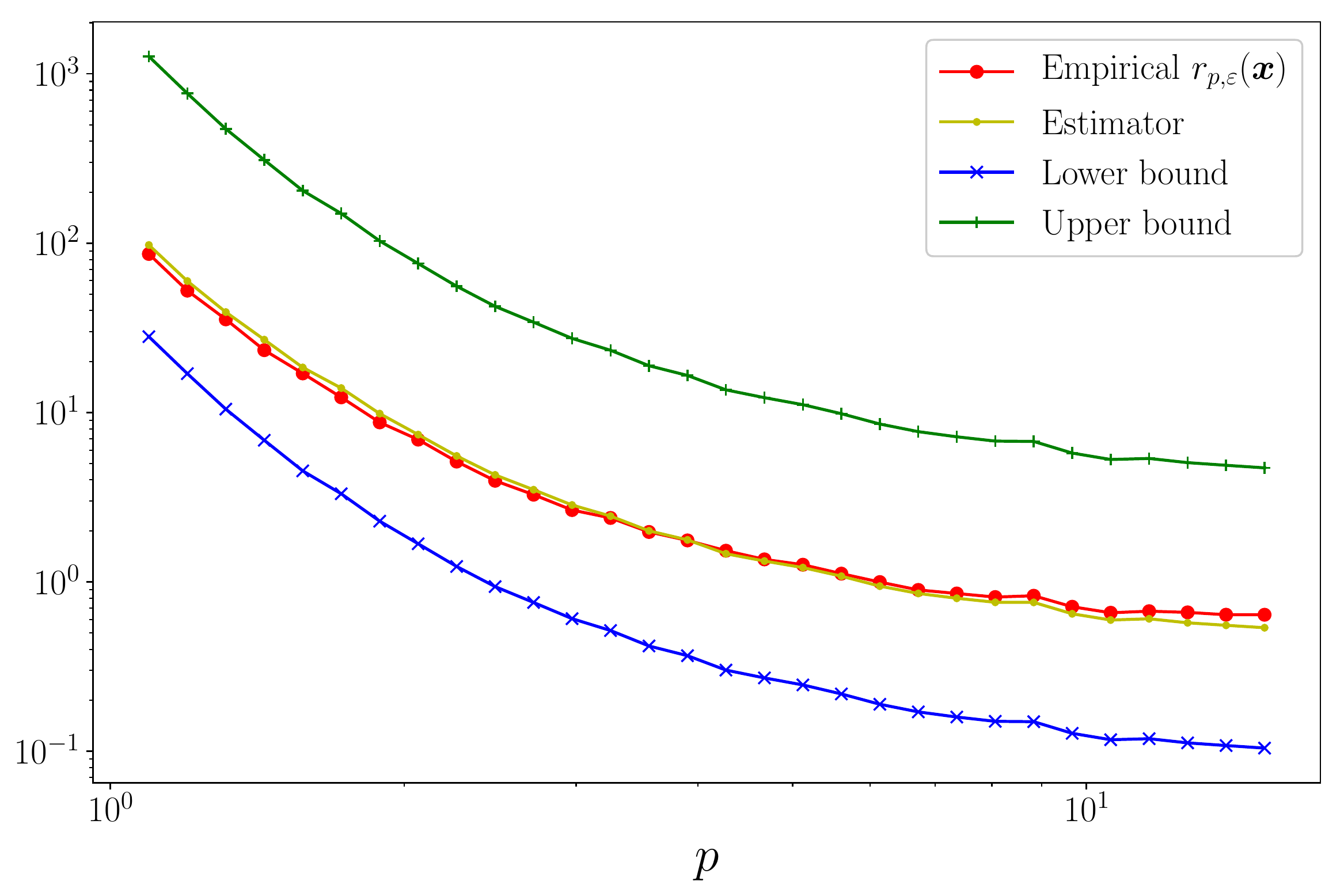}
	\caption{\label{fig:MNIST_linear} Empirical robustness to random uniform noise, derived upper and lower bounds (Theorem~\ref{thm:lp_noise}) and estimate from Eq.~(\ref{eq:estimate}), as a function of $p$ for a linear classifier trained on MNIST. For a given $p$, empirical robustness was computed through an exhaustive search of the smallest radius of the ball where an $\varepsilon$ fraction of points sampled uniformly from the ball are misclassified. We choose $\varepsilon = 1.5\%$, empirically find $\zeta_0 \approx 0.72$, and run the experiments for each chosen $p$ over 1,000 random images from the MNIST test set.}
\end{figure}

\begin{figure}
	\centering
	\includegraphics[width=0.8\textwidth]{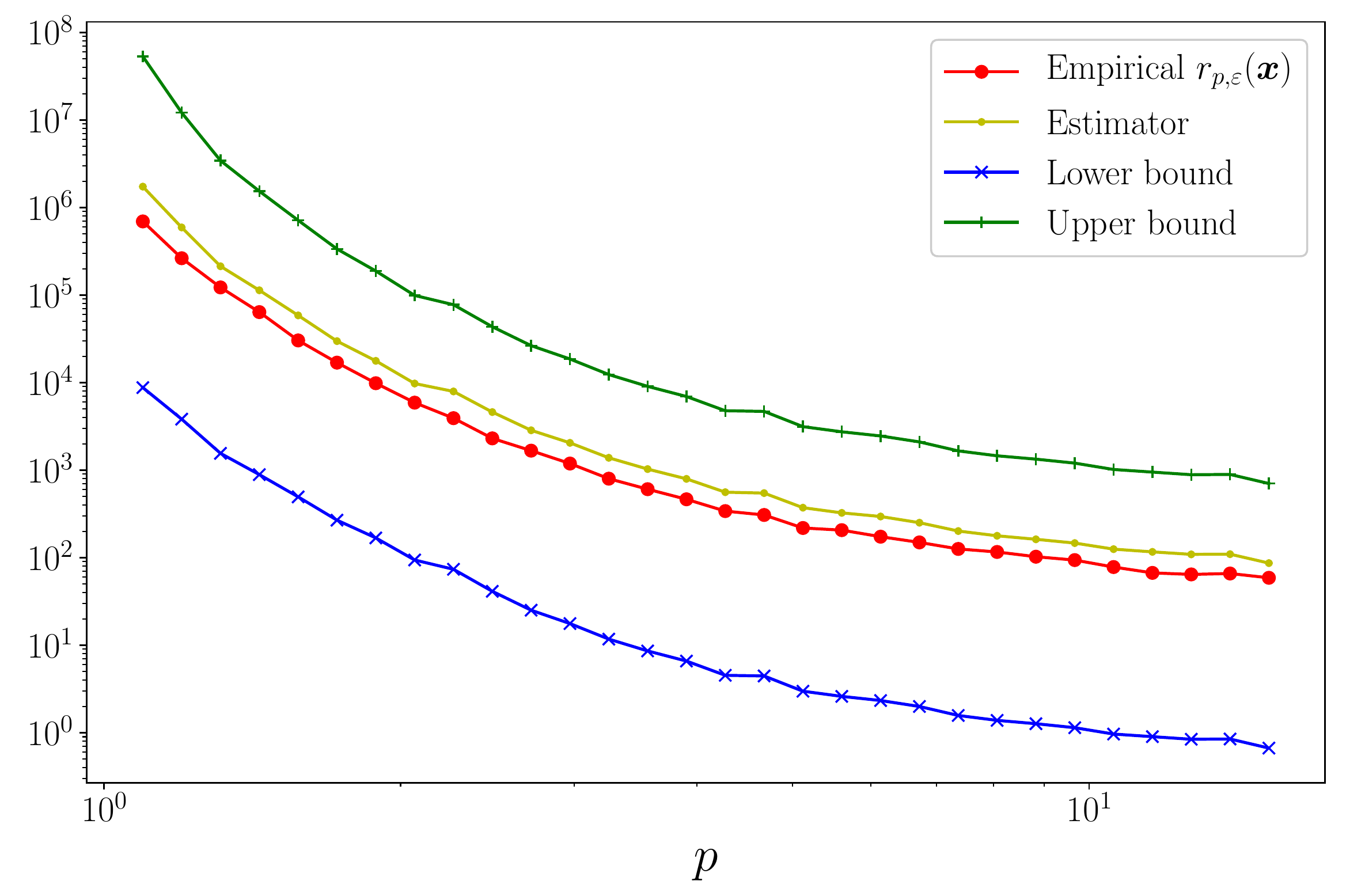}
	\caption{\label{fig:ImageNet_VGG} Empirical robustness to random uniform noise, derived upper and lower bounds (Theorem~\ref{thm:lp_noise_laf}) and estimate from Eq.~(\ref{eq:estimate}), as a function of $p$ for the VGG-19 classifier trained on ImageNet. See the caption of Figure~\ref{fig:MNIST_linear} for more details about the computation of $r_{p,\varepsilon}(\x)$. We choose $\varepsilon = 1.5\%$, use $\zeta_0 \approx 0.72$ as in Figure~\ref{fig:MNIST_linear}, and run the experiments for each chosen $p$ over 200 images from the ImageNet validation set.}
\end{figure}

\textbf{Robustness of a deep neural network to quantization.} We now leverage our analytical results to assess the robustness of a deep neural network classifier to image quantization. When a signal $\x$ is quantized into a discrete valued-signal $Q(\x)$, the quantization noise $Q(\x) - \x$ is often modeled as a signal independent uniform random variable \cite[Chapter~4.5]{bovik2010handbook}. That is, under this assumption, $Q(\x) - \x$ is uniformly distributed over $\mathcal{B}_{\infty} (\bs{0}, \Delta/2)$, with $\Delta$ denoting the quantization step size. According to our analytical results in Section~\ref{sec:robustness_nonlinear}, the approximate step size $\Delta$ that the classifier can tolerate (without changing the estimated label of the quantized image) with probability $1-\varepsilon$ is thus given by:
\[
\Delta = \frac{2 \zeta_0}{\sqrt{\pi}} \sqrt{d} \| \r_{\infty}^*(\x) \|_{\infty},
\]
using the estimate of Eq.~(\ref{eq:estimate}). Moreover, the number of quantization levels required to guarantee robustness of the classifier is therefore estimated by
\begin{align}
\label{eq:prediction_L}
L_q = \frac{255}{\frac{2 \zeta_0}{\sqrt{\pi}} \sqrt{d} \| \r^*_{\infty}(\x) \|_{\infty}}.
\end{align}
In other words, Eq.~(\ref{eq:prediction_L}) predicts that images encoded with more than $\log_2 \left( \frac{255}{\frac{2 \zeta_0}{\sqrt{\pi}} \sqrt{d} \| \r^*_{\infty}(\x) \|_{\infty}} \right)$ bits will have the same estimated label as the original image with high probability, despite quantization. Figure~\ref{fig:quantization} shows that this prediction is a good approximation of the real quantization level computed for $8,000$ images from the ImageNet validation set for the VGG-19 classifier. In this experiment, we use a minimum variance quantization. Moreover, as commonly done, dithering is also applied to improve the perceptual quality of the quantized image.  Interestingly, as predicted by our analysis, most images can be heavily quantized (with e.g., $3$ bits) without changing the label of the classifier, despite the significant distortions to the images caused by heavy quantization (see Figure~\ref{fig:quantization_image} for example images). Finally, note that our analytical results confirm and quantify earlier empirical observations that highlighted the high robustness of classifiers to compression mechanisms \cite{dodge2016understanding, paola1995effect}.

\begin{figure}
	\centering
	\includegraphics[width=0.7\textwidth]{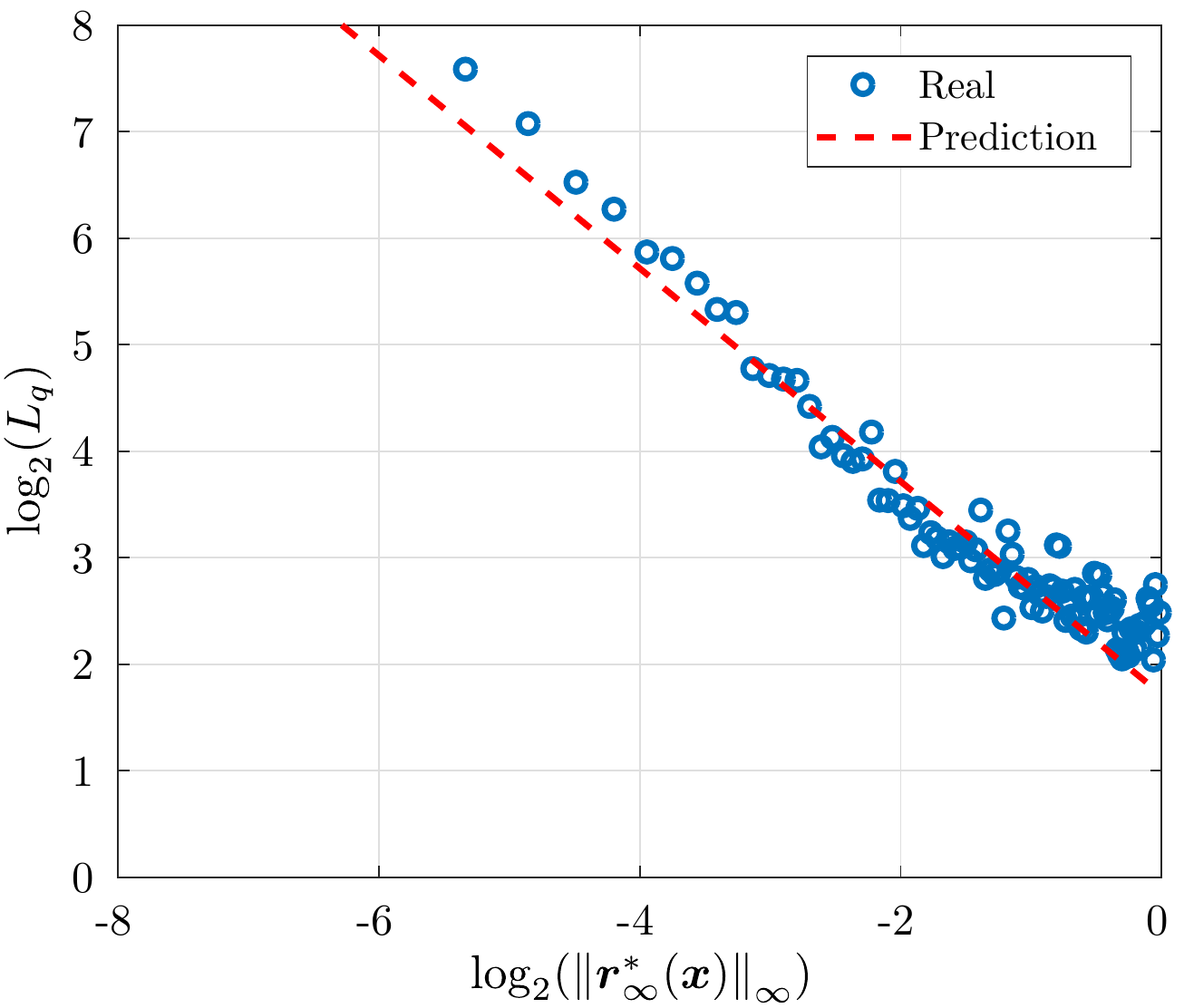}
	\caption{\label{fig:quantization} Minimum number of bits required to encode an image to guarantee similar estimated label as original image vs. $\log_2(\|\r^*_{\infty} (\x)\|_{\infty})$. \textit{Real} points are computed through an exhaustive search of the required quantization level (with different images), and \textit{Prediction} is computed using Eq.~(\ref{eq:prediction_L}). We choose $\varepsilon=1.5\%$ and $\zeta_0=0.72$ as in Figure~\ref{fig:ImageNet_VGG}.}
\end{figure}

\begin{figure}
	\centering
	\includegraphics[width=0.85\textwidth]{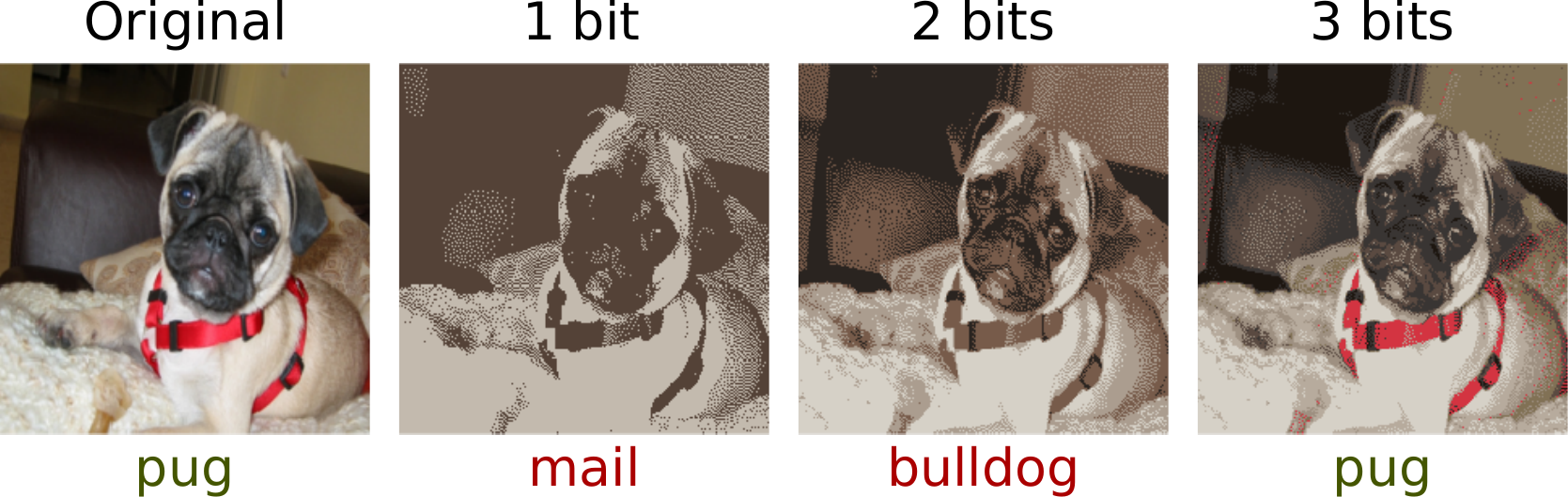}
	\caption{\label{fig:quantization_image}Example image, where a quantization (with dithering) using $3$ bits leads to correct classification.}
\end{figure}

\textbf{Robustness to signal-dependent Gaussian noise.} We now consider the case where some Gaussian noise that correlates with the input image is added to this image. That is, we consider a Gaussian noise $\mathcal{N}(0, \Sigma(\x))$, where $\Sigma(\x)$ is a diagonal matrix such that $\Sigma(\x)_{ii} = 1_{x_i \geq t} \cdot x_i$, where $x_i$ denotes the value of pixel $i$, and $t$ denotes a user-specified threshold.\footnote{We consider in practice color images; the quantity $x_i$ refers in this case to $x_{i, r} + x_{i, g} + x_{i, b}$, where $x_{i, r}, x_{i, g}, x_{i, b}$ respectively denote the red, green and blue channels.} $\Sigma(\x)$ is further normalized to satisfy $\text{Tr} (\Sigma(\x)) = 1$. Under this noise model, noise is solely added to pixels that are ``almost white'' (i.e., pixels satisfying $x_i \geq t$), while all other pixels are left untouched. It should be noted that such signal-dependent noise models are commonly used to model physical deficiencies in acquisition, such as shot noise.

Our analytical results for the Gaussian case predict that the robustness to such noise (provided the gradient directions are ``typical'') should be independent of the distribution of eigenvalues $\Sigma(\x)$, and should moreover satisfy\footnote{We stress here that, due to the normalization $\tr (\Sigma(\x)) = 1$, the same amount of noise is added to all images. It is only the distribution of noise that differs: noise is concentrated on few pixels for images with few white pixels, and spread for white images.}
\begin{align}
\label{eq:exp_gaussian}
\frac{1}{2} \sqrt{d} \leq \frac{r_{\Sigma(\x),\varepsilon} (\x)}{\| \r_2^*(\x) \|_2} \leq 2 \sqrt{d},
\end{align}
where $\varepsilon=0.15$. To verify this hypothesis, we show in Figure~\ref{fig:white_gaussian_exp} the ratio $\frac{r_{\Sigma(\x),\varepsilon} (\x)}{\| \r_2^*(\x) \|_2}$ over 30,000 images from the ImageNet validation set for the VGG-19 classifier, as a function of the ``whiteness'' of the image; i.e., $W(\x) = \sum_i 1_{x_i \geq t} x_i$. Similarly to previous experiments, $r_{\Sigma,\varepsilon} (\x)$ is estimated using an exhaustive line search. It can be seen that the ratio $\frac{r_{\Sigma(\x),\varepsilon} (\x)}{\| \r_2^*(\x) \|_2}$ approximately satisfies the bounds in Eq.~(\ref{eq:exp_gaussian}), although the empirical ratio can surpass the upper bound, for images with significant white pixels. This is potentially due to our assumption on the randomness of the direction of the decision boundary, which can be violated in this case: in fact, white pixels (i.e., non-zero eigenvectors of $\Sigma$) appear often in the background of images, and are thus correlated with the decision boundaries of the classifier. Despite this assumption not being satisfied,  our bounds allow us to predict fairly accurately the behavior of a complex deep network in presence of image-dependent Gaussian noise.


\begin{figure}
\centering
\includegraphics[width=0.7\textwidth]{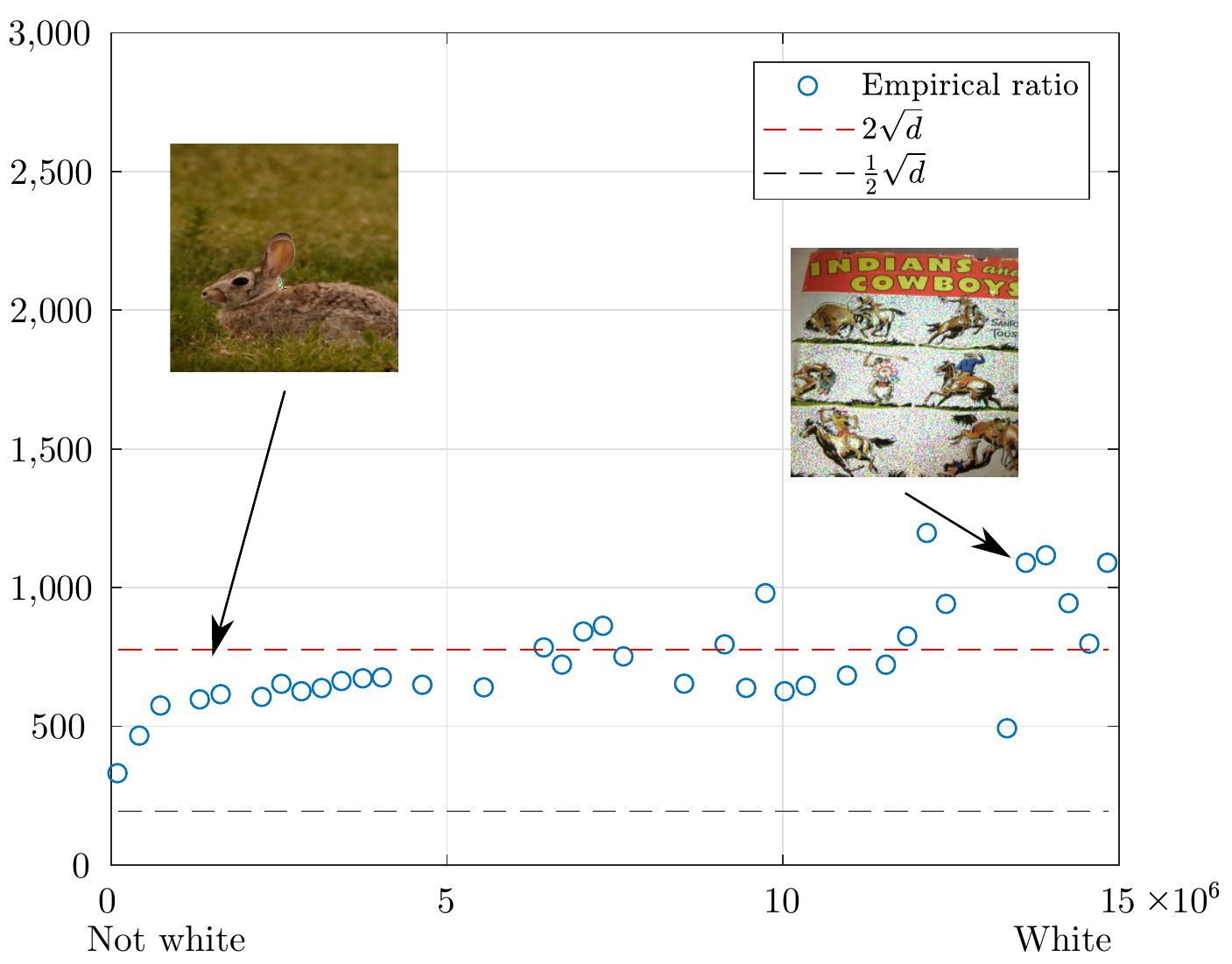}
\caption{\label{fig:white_gaussian_exp} Fraction $\frac{r_{\Sigma(\x),\varepsilon} (\x)}{\| \r_2^*(\x) \|_2}$, where $\varepsilon=15\%$, as a function of $W(\x)$. $W(\x)$ encodes how ``white'' the pixels of the image are. Under our noise model, images with small $W(\x)$ will have the noise concentrated along a few pixels, while images with large $W(\x)$ will have their noise spread across most pixels in the image. Each circle represents the   average robustness ratio of images with the same $W(\x)$.
}
\end{figure}







\section{Conclusion}

We have derived precise bounds on the robustness of linear and nonlinear classifiers to random noise, under two noise distributions: uniform noise in the $\ell_p$ unit ball, and Gaussian noise. Our quantitative results show that state-of-the-art classifiers are orders of magnitude more robust to typical random noise than to worst-case perturbations, typically of order the square root of the input dimension. Such bounds are shown to hold in challenging settings, where a state-of-the-art deep network is used on a large scale multi-class dataset such as ImageNet. Our analysis can be leveraged to quantify the effect of many disturbances (e.g., image quantization) on classifiers, and provide robustness guarantees when such systems are deployed in real world environments. Moreover, our analysis allows us to draw links between different noise regimes, and show the effect of the robustness to adversarial perturbations (or equivalently, the distance to the decision boundary) on other noise regimes.

In this work, we have studied the robustness with respect to generic $\ell_p$ norms. For future work, we believe it would be very interesting to characterize the robustness of classifiers to random perturbations by using perceptual similarity metrics adapted to different modalities, such as images \cite{wang2004image} and speech.


\section*{Acknowledgements}
We gratefully acknowledge the support of NVIDIA Corporation with the
donation of the Titan Xp GPU used for this research.

\small{
	\bibliographystyle{apalike}
	\bibliography{bibliography}
}

\pagebreak

\appendix

\setcounter{theorem}{0}
\setcounter{proposition}{0}

In these appendices, we prove the theoretical results stated in the main article.

\section{Preliminary Results}
In this section, we explicitly compute $\| \r_{p}^*(\x) \|_p$ for a linear classifier as described in the main article.

\begin{lemma}
\label{lem:DistanceHyperplanePNorm}For all $p \in \left[1,\infty\right]$, the $\ell_{p}$-distance from any point $\x$ to the decision hyperplane $\mathcal{H}$ defined by $f\left(\z\right)=0$ is:
\begin{itemize}
	\item if $p=\infty$:
	$$\left\| \r_{\infty}^*(\x) \right\|_{\infty} =\frac{\left|f\left(\x\right)\right|}{\left\Vert \w\right\Vert _{1}}\text{;}$$
	\item if $p=1$:
	$$\left\| \r_1^*(\x) \right\|_{1} =\frac{\left|f\left(\x\right)\right|}{\left\Vert \w\right\Vert _{\infty}}\text{;}$$
	\item if $p \in \left(1,\infty\right)$:
	$$\| \r_{p}^*(\x) \|_p=\frac{\left|f\left(\x\right)\right|}{\left\Vert \w\right\Vert _{\frac{p}{p-1}}}\text{.}$$
\end{itemize}
Overall, for all $p\in\left[1,\infty\right]$,
the $\ell_{p}$-distance from any point $\x$ to the decision hyperplane
$\mathcal{H}:f\left(\z\right)=0$ is:
\begin{equation*}
\| \r_{p}^*(\x) \|_p=\frac{\left|f\left(\x\right)\right|}{\left\Vert \w\right\Vert _{\frac{p}{p-1}}}\text{.}
\end{equation*}
\end{lemma}
\begin{proof}
We distinguish between the three cases.
\begin{itemize}
	\item Suppose $p=\infty$.
	The distance from $\x$ to $\mathcal{H}$ is equal to the minimum radius
	$\alpha$ of a ball (i.e., for $\ell_{\infty}$, a hypercube)
	centered at $\x$ that intersects $\mathcal{H}$. This intersection
	with minimum radius necessarily contains a vertex of the hypercube.
	To determine which one, it suffices to determine which vector $\x+\alpha\bs{\varepsilon}$,
	with $\bs{\varepsilon}\in\left\{ -1,1\right\} ^{d}$, first intersects $\mathcal{H}$
	when $\alpha$ increases starting at $0$. Such an intersection arises
	when $\w^{T}\left(\x+\alpha\bs{\varepsilon}\right)+\bs{b}=0$, so $\alpha=-\frac{f\left(\x\right)}{\w^{T}\varepsilon}$,
	and since $\alpha$ must be non-negative:
	\[
	\r_{\infty}^*(\x)=\min_{f\left(\x\right)\cdot \w^{T}\bs{\varepsilon}\leq0}-\frac{f\left(\x\right)}{\w^{T}\bs{\varepsilon}}=\frac{\left|f\left(\x\right)\right|}{\left\Vert \w\right\Vert _{1}}\text{,}
	\]
	because $\bs{\varepsilon}\in\left\{ -1,1\right\} ^{d}$ (simply choose $\varepsilon_{i}=\sign\left(-f\left(\x\right)w_{i}\right)$).

	\item Suppose $p=1$.
	In this case, the proof is symmetric to the one for $p=\infty$,
	with $\bs{\varepsilon}\in\left\{ -1,1\right\} ^{d}$ having exactly one non-zero
	coordinate.

	\item Suppose $p \in \left(1,\infty\right)$.
	The distance from $\x$ to $\mathcal{H}$ is equal to the minimum radius
	$\alpha$ of an $\ell_{p}$ ball $\mathcal{B}_{p}$ centered at $\x$
	that intersects $\mathcal{H}$. This ball is described by the following
	equation (where $\bs{z}$ is the variable):
	\[
	\sum_{i=1}^{d}\left(z_{i}-x_{i}\right)^{p}\leq\alpha^{p}\text{.}
	\]
	For such a minimum radius, the plane described by $f\left(\z\right)=0$
	is tangent to $\mathcal{B}_{p}$ at some point $\x+\bs{n}$. Let us assume
	without loss of generality that every coordinate of $\bs{n}$ is non-negative.
	We also know that this hyperplane is described by the following equations (where $\bs{z}$ is the variable):
	\begin{eqnarray*}
	\nabla_{\x+\bs{n}}\left(\sum_{i=1}^{d}\left(z_{i}-x_{i}\right)^{p}-\alpha^{p}\right)^{T}\left(\z-\left(\x+\bs{n}\right)\right)=0 & \Leftrightarrow & \sum_{i=1}^{d}n_{i}^{p-1}\left(z_{i}-x_{i}-n_{i}\right)=0\\
	 & \Leftrightarrow & \sum_{i=1}^{d}n_{i}^{p-1}\left(z_{i}-x_{i}\right)=\alpha^{p}\text{,}
	\end{eqnarray*}
	beacuse $\bs{n}$ belongs to the boundary of $\mathcal{B}_p$.
	The last equation thus describes the same hyperplane as $\w^{T} \z=-\bs{b}$.
	Therefore, there exists $\lambda\in\mathbb{R}\setminus\left\{ 0\right\} $
	such that $\forall i,n_{i}^{p-1}=\lambda w_{i}$. Then, since $\x+\bs{n}\in\mathcal{B}_{p}$:
	\[
	\sum_{i=1}^{d}n_{i}^{p-1}\left(\left(x_{i}+n_{i}\right)-x_{i}\right)=\lambda\sum_{i=1}^{d}w_{i}n_{i}=\alpha^{p}\text{,}
	\]
	and, since $\x+\bs{n}\in\mathcal{H}$:
	\[
	\sum_{i=1}^{d}w_{i}\left(x_{i}+n_{i}\right)+\bs{b}=f\left(x\right)+\w^{T}\bs{n}=0\text{,}
	\]
	we have $\lambda=-\frac{\alpha^{p}}{f\left(\x\right)}$.
	Finally:
	\begin{eqnarray*}
	\alpha & = & \left(\sum_{i=1}^{d}n_{i}^{p}\right)^{\frac{1}{p}}=\left(\sum_{i=1}^{d}\left(\lambda w_{i}\right)^{\frac{p}{p-1}}\right)^{\frac{1}{p}}=\left(\frac{\alpha^{p}}{\left|f\left(\x\right)\right|}\right)^{\frac{1}{p-1}}\left\Vert \w\right\Vert _{\frac{p}{p-1}}^{\frac{1}{p-1}}\\
	\alpha & = & \frac{\left|f\left(\x\right)\right|}{\left\Vert \w\right\Vert _{\frac{p}{p-1}}}\text{.}
	\end{eqnarray*}
\end{itemize}
\end{proof}

\section{Robustness of Linear Classifiers to $\ell_p$ Noise}

\subsection{Main Theorem}

\begin{theorem}
\label{thm:lp_noise2}
Let $p \in [1, \infty]$. Let $p' \in [1,\infty]$ be such that $\frac{1}{p} + \frac{1}{p'} = 1$. Then there exist universal constants $C, c, c' > 0$ such that, for all $\varepsilon < \frac{c^2}{c'}$:
\begin{equation*}
\zeta_1(\varepsilon) d^{1/p} \frac{ \| \w \|_{p'} }{\| \w \|_{2} } \leq \frac{r_{p,\varepsilon}(\x)}{\| \r^*_{p}(\x) \|_p} \leq \zeta_2(\varepsilon) d^{1/p} \frac{ \| \w \|_{p'} }{\| \w \|_{2} },
\end{equation*}
where $\zeta_1(\varepsilon) = C \sqrt{\varepsilon}$ and $\zeta_2(\varepsilon) = \frac{1}{\sqrt{c - \sqrt{c' \varepsilon}}}$.
\end{theorem}

Theorem~\ref{thm:lp_noise2} is proved by the following lemmas.

\begin{lemma}
\label{lem:lp_noise_part1}
There exists a universal constant $C > 0$ such that
\begin{equation*}
\frac{r_{p,\varepsilon}(\x)}{\| \r^*_{p}(\x) \|_p} \geq \zeta_1(\varepsilon) d^{1/p} \frac{ \| \w \|_{p'} }{\| \w \|_{2} },
\end{equation*}
where $\zeta_1(\varepsilon) = C \sqrt{\varepsilon}$.
\end{lemma}
\begin{proof}
Let us first express conveniently $\bb{P}_{\vb \sim \calB_p}\left\{g\left(\x\right) \neq g\left(\x+\alpha \vb\right)\right\}$, where $\vb \sim \calB_p$ means that $\vb$ is chosen uniformly at random in $\calB_p$:
\begin{eqnarray}
\bb{P}_{\vb \sim \calB_p}\left\{g\left(\x\right) \neq g\left(\x+\alpha \vb\right)\right\} & = & \bb{P}_{\vb \sim \calB_p}\left(f\left(\x\right)f\left(\x+\alpha \vb\right)\leq0\right\}\nonumber \\
 & = & \bb{P}_{\vb \sim \calB_p}\left\{\sign\left(f\left(\x\right)\right)\left(\w^{T}x+\bs{b}\right)\leq-\sign\left(f\left(\x\right)\right)\alpha\w^{T}\vb\right\}\nonumber \\
 & = & \bb{P}_{\vb \sim \calB_p}\left\{\left\Vert \w\right\Vert _{\frac{p}{p-1}}\| \r_{p}^*(\x) \|_p\leq-\sign\left(f\left(x\right)\right)\alpha\w^{T}\vb\right\}\label{eq:ApplicationDistancePointPlane}\\
 & = & \bb{P}_{\vb \sim \calB_p}\left\{\left\Vert \w\right\Vert _{p'}\frac{\| \r_{p}^*(\x) \|_p}{\left|\alpha\right|}\leq \w^{T}\vb\right\}\label{eq:SymmetryUniformDistribution}\\
 & = & \frac{1}{2}\mathbb{P}_{\vb \sim \calB_p}\left\{\left\Vert \w\right\Vert _{p'}\frac{\| \r_{p}^*(\x) \|_p}{\left|\alpha\right|}\leq\left|\w^{T}\vb\right|\right\},\label{eq:SymmetryUniformDistribution2}
\end{eqnarray}
where Eq.~(\ref{eq:ApplicationDistancePointPlane}) is given by Lemma~\ref{lem:DistanceHyperplanePNorm},
and Eq.~(\ref{eq:SymmetryUniformDistribution}) and (\ref{eq:SymmetryUniformDistribution2}) follow from $\vb \sim \calB_p\Rightarrow-\vb \sim \calB_p$.

\medskip{}
Markov's inequality gives, from Eq.~(\ref{eq:SymmetryUniformDistribution2}):
\[
\bb{P}_{\vb \sim \calB_p}\left\{g\left(\x\right) \neq g\left(\x+\alpha \vb\right)\right\}\leq \frac{1}{2}\frac{\mathbb{E}_{\vb \sim \calB_p}\left[\left(\sum_{i=1}^{d}w_{i}v_{i}\right)^{2}\right]}{\left(\left\Vert \w\right\Vert _{p'}\frac{\| \r_{p}^*(\x) \|_p}{\left|\alpha\right|}\right)^{2}}\text{.}
\]
In \cite[Theorem\ 7]{barthe2005probabilistic}, it is proved
that there is a constant $C_{0}>0$ such that:
\[
\mathbb{E}_{\vb \sim \calB_p}\left[\left(\sum_{i=1}^{d}w_{i}v_{i}\right)^{2}\right]\leq\left(\frac{2C_{0}}{d^{\frac{1}{p}}}\left\Vert \w\right\Vert _{2}\right)^{2}\text{,}
\]
Therefore:
\[
\bb{P}_{\vb \sim \calB_p}\left\{g\left(\x\right) \neq g\left(\x+\alpha \vb\right)\right\}\leq \frac{1}{2}\frac{\left(\frac{2C_{0}}{d^{\frac{1}{p}}}\left\Vert \w\right\Vert _{2}\right)^{2}}{\left(\left\Vert \w\right\Vert _{p'}\frac{\| \r_{p}^*(\x) \|_p}{\left|\alpha\right|}\right)^{2}}\text{\text{.}}
\]
So, if $\left|\alpha\right|<\sqrt{\varepsilon}\frac{d^{\frac{1}{p}}}{\sqrt{2}C_{0}}\frac{\left\Vert \w\right\Vert _{p'}}{\left\Vert \w\right\Vert _{2}}\| \r_{p}^*(\x) \|_p$,
then $\bb{P}_{\vb \sim \calB_p}\left\{g\left(\x\right) \neq g\left(\x+\alpha \vb\right)\right\}<\varepsilon$.
Thus, there is a universal constant $C=\frac{1}{\sqrt{2}C_{0}}>0$ such that:
\begin{equation*}
\frac{r_{p,\varepsilon}(\x)}{\| \r^*_{p}(\x) \|_p} \geq \zeta_1(\varepsilon) d^{1/p} \frac{ \| \w \|_{p'} }{\| \w \|_{2} }.
\end{equation*}
\end{proof}

\begin{lemma}
\label{lem:lp_noise_part2}
There exist universal constants $c, c' > 0$ such that, for all $\varepsilon < \frac{c^2}{c'}$:
\begin{equation*}
\frac{r_{p,\varepsilon}(\x)}{\| \r^*_{p}(\x) \|_p} \leq \zeta_2(\varepsilon) d^{1/p} \frac{ \| \w \|_{p'} }{\| \w \|_{2} },
\end{equation*}
where $\zeta_2(\varepsilon) = \frac{1}{\sqrt{c - \sqrt{c' \varepsilon}}}$.
\end{lemma}
\begin{proof}
We first transform the expression of $\bb{P}_{\vb \sim \calB_p}\left\{g\left(\x\right) \neq g\left(\x+\alpha \vb\right)\right\}$:
\begin{eqnarray*}
\bb{P}_{\vb \sim \calB_p}\left\{g\left(\x\right) \neq g\left(\x+\alpha \vb\right)\right\} & = & \mathbb{P}_{\vb \sim \calB_p}\left\{\left\Vert \w\right\Vert _{\frac{p}{p-1}}\frac{\| \r_{p}^*(\x) \|_p}{\left|\alpha\right|}\leq \w^{T}\vb\right\}\\
 & = & \frac{1}{2}\mathbb{P}_{\vb \sim \calB_p}\left\{\left\Vert \w\right\Vert _{p'}\frac{\| \r_{p}^*(\x) \|_p}{\left|\alpha\right|}\leq\left|\w^{T}\vb\right|\right\}\\
 & = & \frac{1}{2}\mathbb{P}_{\vb \sim \calB_p}\left\{\frac{1}{\var\left(\w^{T}\vb\right)}\left(\left\Vert \w\right\Vert _{p'}\frac{\| \r_{p}^*(\x) \|_p}{\left|\alpha\right|}\right)^{2}\leq\frac{\left(\w^{T}\vb\right)^{2}}{\var\left(\w^{T}\vb\right)}\right\}.
\end{eqnarray*}
Paley-Zygmund's inequality states that, if $X$ is a random variable with finite variance and $t\in\left[0,1\right]$, then:$$\bb{P}\left\{X > t \bb{E}\left[X\right]\geq(1-t)^2 \frac{\bb{E}\left[X\right]^2}{\bb{E}\left[X^2\right]}\right\}.$$
Note that $\mathbb{E}_{\vb \sim \calB_p}\left(\frac{\left(\w^{T}\vb\right)^{2}}{\var\left(\w^{T}\vb\right)}\right)=1$,
because $\mathbb{E}_{\vb \sim \calB_p}\left(\w^{T}\vb\right)=0$.
So, by using Paley-Zygmund's inequality with $X=\frac{\left(\w^{T}\vb\right)^{2}}{\var\left(\w^{T}\vb\right)}$ and $t=\frac{1}{\var\left(\w^{T}\vb\right)}\left(\left\Vert \w\right\Vert _{p'}\frac{\| \r_{p}^*(\x) \|_p}{\left|\alpha\right|}\right)^{2}$, when $\left|\alpha\right|\geq\frac{\left\Vert \w\right\Vert _{p'}}{\sqrt{\var\left(\w^{T}\vb\right)}}\| \r_{p}^*(\x) \|_p$:
\begin{eqnarray*}
\bb{P}_{\vb \sim \calB_p}\left\{g\left(\x\right) \neq g\left(\x+\alpha \vb\right)\right\} & \geq & \frac{\left(1-\frac{1}{\var\left(\w^{T}\vb\right)}\left(\left\Vert \w\right\Vert _{p'}\frac{\| \r_{p}^*(\x) \|_p}{\left|\alpha\right|}\right)^{2}\right)^{2}}{2\mathbb{E}_{\vb \sim \calB_p}\left[\frac{\left(\w^{T}\vb\right)^{4}}{\var\left(\w^{T}\vb\right)^{2}}\right]}\text{.}
\end{eqnarray*}
So, if $\left|\alpha\right|>\frac{1}{\sqrt{\var\left(\w^{T}\vb\right)-\sqrt{2\varepsilon\mathbb{E}\left[\left(\w^{T}\vb\right)^{4}\right]}}}\left\Vert \w\right\Vert _{p'}\| \r_{p}^*(\x) \|_p$,
then $\bb{P}_{\vb \sim \calB_p}\left\{g\left(\x\right) \neq g\left(\x+\alpha \vb\right)\right\}>\varepsilon$.
According to \cite[Theorem\ 7]{barthe2005probabilistic}, there is a universal constant $c_{0}>0$ such that:
\begin{itemize}
\item for $\var\left(\w^{T}\vb\right)$:
\[
\var\left(\w^{T}\vb\right)\geq\left(\frac{c_{0}}{d^{\frac{1}{p}}}\left\Vert \w\right\Vert _{2}\right)^{2}\text{;}
\]
\item for $\mathbb{E}\left[\left(\w^{T}\vb\right)^{4}\right]$:
\[
\mathbb{E}\left[\left(\w^{T}\vb\right)^{4}\right]\leq\left(\frac{4C_{0}}{d^{\frac{1}{p}}}\left\Vert \w\right\Vert _{2}\right)^{4}\text{.}
\]
\end{itemize}
So there are universal constants $c=c_{0}^{2},c'=512C_{0}^{4}>0$
such that:
\begin{equation*}
\frac{r_{p,\varepsilon}(\x)}{\| \r^*_{p}(\x) \|_p} \leq \zeta_2(\varepsilon) d^{1/p} \frac{ \| \w \|_{p'} }{\| \w \|_{2} }\text{.}
\end{equation*}
\end{proof}

\subsection{Alternative Lower Bound}
Actually, the lower bound of Theorem~\ref{thm:lp_noise2} may be improved for most $p$-norms by the following result.

\begin{lemma}
\label{lem:lp_noise_alternative}
There exists a universal constant $C' > 0$ such that
\begin{equation*}
\frac{r_{p,\varepsilon}(\x)}{\| \r^*_{p}(\x) \|_p} \geq \zeta_1(\varepsilon) d^{1/p} \frac{ \| \w \|_{p'} }{\| \w \|_{2} },
\end{equation*}
where $\zeta_1(\varepsilon) = \frac{C'}{\sqrt{\log \frac{3}{\varepsilon}}} \left(1-\frac{1}{\min(p, 2)}\right)$.
\end{lemma}
\begin{proof}
Let $p_{2}=\min\left(p,2\right)$. We have:
\begin{eqnarray*}
\bb{P}_{\vb \sim \calB_p}\left\{g\left(\x\right) \neq g\left(\x+\alpha \vb\right)\right\} & = & \mathbb{P}_{\vb \sim \calB_p}\left\{\left\Vert \w\right\Vert _{\frac{p}{p-1}}\frac{\| \r^*_{p}(\x) \|_p}{\left|\alpha\right|}\leq \w^{T}\vb\right\}\\
 & = & \mathbb{P}_{\vb \sim \calB_p}\left\{e^{\theta t}\leq\exp\left(\theta \w^{T}\vb\right)\right\}\text{,}
\end{eqnarray*}
where $t=\left\Vert \w\right\Vert _{p'}\frac{\| \r^*_{p}(\x) \|_p}{\left|\alpha\right|}$,
for any $\theta>0$. Markov's inequality gives:
\begin{eqnarray*}
\bb{P}_{\vb \sim \calB_p}\left\{g\left(\x\right) \neq g\left(\x+\alpha \vb\right)\right\} & \leq & \frac{1}{e^{\theta t}}\mathbb{E}_{\vb \sim \calB_p}\left[\exp\left(\theta \w^{T}\vb\right)\right]=\frac{1}{e^{\theta t}}\sum_{k=0}^{\infty}\frac{1}{k!}\mathbb{E}_{\vb \sim \calB_p}\left[\left(\theta \w^{T}\vb\right)^{k}\right]\text{.}\\
 & \leq & \frac{1}{e^{\theta t}}\sum_{k=0}^{\infty}\frac{1}{\left(2k\right)!}\mathbb{E}_{\vb \sim \calB_p}\left[\left(\theta \w^{T}\vb\right)^{2k}\right]\text{,}
\end{eqnarray*}
since $\w^{T}\vb$ is symmetric. In \cite[Theorem\ 7]{barthe2005probabilistic},
it is proved that:
\begin{itemize}
\item if $k\leq d$ and $p\leq2$:
\[
\mathbb{E}_{\vb \sim \calB_p}\left[\left|\sum_{i=1}^{d}w_{i}v_{i}\right|^{k}\right]\leq\left(\frac{C_{0}k^{\frac{1}{p}}}{d^{\frac{1}{p}}}\left\Vert \w\right\Vert _{2}\right)^{k}\text{;}
\]
\item if $k\leq d$ and $p>2$:
\[
\mathbb{E}_{\vb \sim \calB_p}\left[\left|\sum_{i=1}^{d}w_{i}v_{i}\right|^{k}\right]\leq\left(\frac{C_{0}k^{\frac{1}{2}}}{d^{\frac{1}{p}}}\left\Vert \w\right\Vert _{2}\right)^{k}\text{;}
\]
\item if $k>d$ and $p\leq2$:
\[
\mathbb{E}_{\vb \sim \calB_p}\left[\left|\sum_{i=1}^{d}w_{i}v_{i}\right|^{k}\right]\leq\left(C_{0}\left\Vert \w\right\Vert _{2}\right)^{k}\leq\left(\frac{C_{0}k^{\frac{1}{p}}}{d^{\frac{1}{p}}}\left\Vert \w\right\Vert _{2}\right)^{k}\text{;}
\]
\item if $k>d$ and $p>2$:
\[
\mathbb{E}_{\vb \sim \calB_p}\left[\left|\sum_{i=1}^{d}w_{i}v_{i}\right|^{k}\right]\leq\left(\frac{C_{0}d^{\frac{1}{2}}}{d^{\frac{1}{p}}}\left\Vert \w\right\Vert _{2}\right)^{k}\leq\left(\frac{C_{0}k^{\frac{1}{2}}}{d^{\frac{1}{p}}}\left\Vert \w\right\Vert _{2}\right)^{k}\text{,}
\]
\end{itemize}
where $C_{0}$ is a universal constant (the same as in the proof of
Lemma~\ref{lem:lp_noise_part1}). So, overall:
\[
\mathbb{E}_{\vb \sim \calB_p}\left[\left|\sum_{i=1}^{d}w_{i}v_{i}\right|^{k}\right]\leq\left(\frac{C_{0}k^{\frac{1}{p_{2}}}}{d^{\frac{1}{p}}}\left\Vert \w\right\Vert _{2}\right)^{k}\text{.}
\]
Thus:
\[
\bb{P}_{\vb \sim \calB_p}\left\{g\left(\x\right) \neq g\left(\x+\alpha \vb\right)\right\}\leq\frac{1}{e^{\theta t}}\sum_{k=0}^{\infty}\frac{1}{\left(2k\right)!}\left(\theta\frac{C_{0}\left(2k\right)^{\frac{1}{p_{2}}}}{d^{\frac{1}{p}}}\left\Vert \w\right\Vert _{2}\right)^{2k}\text{.}
\]
We can bound the following power series using Stirling-like bounds
\cite{robbins1955stirling} in (\ref{eq:StirlingLowerBound}) and (\ref{eq:StirlingUpperBound}):
\begin{eqnarray}
\sum_{k=0}^{\infty}\frac{\left(2k\right)^{\frac{2k}{p_{2}}}}{\left(2k\right)!}x^{k} & \leq & 1+\frac{1}{\sqrt{2\pi}}\sum_{k=1}^{\infty}\frac{\left(2k\right)^{\frac{2k}{p_{2}}}}{\left(2k\right)^{2k+\frac{1}{2}}}\left(e^{2}x\right)^{k}\label{eq:StirlingLowerBound}\\
 &  & =1+\frac{1}{\sqrt{2\pi}}\sum_{k=1}^{\infty}\left(2k\right)^{-2\left(1-\frac{1}{p_{2}}\right)k-\frac{1}{2}}\left(e^{2}x\right)^{k}\nonumber \\
 & \leq & 1+\frac{1}{\sqrt{2\pi}}\sum_{k=1}^{\infty}\left(\left\lfloor 2\left(1-\frac{1}{p_{2}}\right)k\right\rfloor \right)^{-\left\lfloor 2\left(1-\frac{1}{p_{2}}\right)k\right\rfloor -\frac{1}{2}}\left(e^{2}x\right)^{k}\nonumber \\
 & \leq & 1+\frac{e}{\sqrt{2\pi}}\sum_{k=1}^{\infty}\frac{\exp\left(-\left\lfloor 2\left(1-\frac{1}{p_{2}}\right)k\right\rfloor \right)}{\left\lfloor 2\left(1-\frac{1}{p_{2}}\right)k\right\rfloor !}\left(e^{2}x\right)^{k}\label{eq:StirlingUpperBound}\\
 & \leq & 1+\frac{e^{2}}{\sqrt{2\pi}}\sum_{k=1}^{\infty}\frac{\exp\left(-2\left(1-\frac{1}{p_{2}}\right)k\right)}{k!}\left(e^{2}x\right)^{k}\text{.}\nonumber
\end{eqnarray}
Therefore:
\[
\bb{P}_{\vb \sim \calB_p}\left\{g\left(\x\right) \neq g\left(\x+\alpha \vb\right)\right\}\leq3e^{-\theta t}\exp\left(\theta^{2}\frac{p_{2}}{p_{2}-1}\frac{e^{2}C_{0}^{2}}{d^{\frac{2}{p}}}\left\Vert \w\right\Vert _{2}^{2}\right)\text{.}
\]
By choosing $\theta=\frac{1}{2}t\left(\frac{p_{2}}{p_{2}-1}\frac{e^{2}C_{0}k^{\frac{1}{p_{2}}}}{d^{\frac{2}{p}}}\left\Vert \w\right\Vert _{2}\right)^{-1}$:
\begin{eqnarray*}
\bb{P}_{\vb \sim \calB_p}\left\{g\left(\x\right) \neq g\left(\x+\alpha \vb\right)\right\} & \leq & 3\exp\left(-t^{2}\left(1-\frac{1}{p_{2}}\right)\frac{d^{\frac{2}{p}}}{2e^{2}C_{0}^{2}\left\Vert \w\right\Vert _{2}^{2}}\right)\\
 &  & =3\exp\left(-\left(\frac{\| \r^*_{p}(\x) \|_p}{\left|\alpha\right|}\right)^{2}\left(1-\frac{1}{p_{2}}\right)\frac{d^{\frac{2}{p}}\left\Vert \w\right\Vert _{p'}^{2}}{2e^{2}C_{0}^{2}\left\Vert \w\right\Vert _{2}^{2}}\right)\text{.}
\end{eqnarray*}
So, if $\left|\alpha\right|<\frac{C'}{\sqrt{\ln\frac{3}{\varepsilon}}}\left(1-\frac{1}{p_{2}}\right)d^{\frac{1}{p}}\frac{\left\Vert \w\right\Vert _{p'}}{\left\Vert \w\right\Vert _{2}}\| \r^*_{p}(\x) \|_p$,
then $\bb{P}_{\vb \sim \calB_p}\left\{g\left(\x\right) \neq g\left(\x+\alpha \vb\right)\right\}<\varepsilon$,
where $C=\frac{1}{2e^{2}C_{0}^{2}}>0$ is a universal constant, and:
\begin{equation*}
\frac{r_{p,\varepsilon}(\x)}{\| \r^*_{p}(\x) \|_p} \geq \zeta_1(\varepsilon) d^{1/p} \frac{ \| \w \|_{p'} }{\| \w \|_{2} }\text{.}
\end{equation*}
\end{proof}

\subsection{Typical Value of the Multiplicative Factor}

\begin{proposition}
\label{prop:MeanBoundsRandom2}
For any $p\in\left(1,\infty\right]$, if $\w$ is a random direction uniformly distributed over the unit
$\ell_{2}$-sphere, then, as $d\to\infty$:
\begin{equation*}
\frac{d^{1/p} \frac{\left\Vert \w\right\Vert _{p'}}{\| \w \|_2}}{\sqrt{d}} \xrightarrow[\text{a.s.}]{}\sqrt{2}\left(\frac{\Gamma\left(\frac{2p-1}{2\left(p-1\right)}\right)}{\sqrt{\pi}}\right)^{1-\frac{1}{p}}.
\end{equation*}
Moreover, for $p = 1$,
\begin{equation*}
\frac{d\frac{\left\Vert \w \right\Vert _{\infty}}{\| \w \|_2}}{ \sqrt{2 d \ln d} }\xrightarrow[\text{a.s.}]{} 1 \text{.}
\end{equation*}
\end{proposition}
\begin{proof}
$\w$ can be written
as $\frac{\g}{\left\Vert \g\right\Vert _{2}}$, where $\g=\left(g_{1},\ldots,g_{d}\right)$
are i.i.d. with normal distribution ($\mu=0$, $\sigma^{2}=\frac{1}{2}$).

The law of large numbers gives that, for $p'\neq\infty$:
\[
\frac{1}{d} \sum_{i=1}^{d}\left|g_{i}\right|^{p'}\xrightarrow[\text{a.s.}]{}\mathbb{E}\left(\left|g_{1}\right|^{p'}\right)= \frac{\Gamma\left(\frac{1+p'}{2}\right)}{\sqrt{\pi}}\text{.}
\]
Thus:
\[
\frac{1}{d^{\frac{1}{p'}}}\left\Vert \g\right\Vert _{p'}\xrightarrow[\text{a.s.}]{}\left(\frac{\Gamma\left(\frac{1+p'}{2}\right)}{\sqrt{\pi}}\right)^{\frac{1}{p'}}\text{\text{,}}
\]
and, for $p\in\left(1,\infty\right]$:
\begin{equation*}
\frac{d^{\frac{1}{p}}}{\sqrt{d}}\left\Vert \frac{\g}{\left\Vert \g\right\Vert _{2}}\right\Vert _{p'}\xrightarrow[\text{a.s.}]{}\sqrt{2}\left(\frac{\Gamma\left(\frac{2p-1}{2\left(p-1\right)}\right)}{\sqrt{\pi}}\right)^{1-\frac{1}{p}}\text{,}
\end{equation*}
because $\frac{\| \g \|_2}{\sqrt{d}} \xrightarrow[\text{a.s.}]{}\frac{1}{\sqrt{2}}$.

For $p=1$ we use a result proved in \cite[Example\ 4.4.1]{galambos1987extremeorder}
directly implying that $$\frac{\left\Vert \g\right\Vert _{\infty}}{\sqrt{\ln d}}\xrightarrow[\text{a.s.}]{}1.$$
Using the previous computations for $p=2$, we find:
\begin{equation*}
\frac{d\frac{\left\Vert \w \right\Vert _{\infty}}{\| \w \|_2}}{ \sqrt{2 d \ln d} }\xrightarrow[\text{a.s.}]{} 1.
\end{equation*}
\end{proof}

\section{Robustness of Linear Classifiers to Gaussian Noise}

\subsection{Main Theorem}

\begin{theorem}
\label{thm:gaussian_noise2}
For $\varepsilon < \frac{1}{3}$, $\zeta_1'(\varepsilon) = \sqrt{\frac{1}{2\ln\left(\frac{1}{\varepsilon}\right)}}$ and $\zeta_2'(\varepsilon) = \sqrt{\frac{1}{1-\sqrt{3\varepsilon}}}$:
\begin{equation*}
\zeta_1'(\varepsilon) \frac{ \| \w \|_{2} }{\| \sqrt{\Sigma} \w \|_{2} } \leq \frac{r_{\Sigma,\varepsilon}(\x)}{\| \r^*_{2}(\x) \|_2} \leq \zeta_2'(\varepsilon)  \frac{ \| \w \|_{2} }{\| \sqrt{\Sigma} \w \|_{2} }.
\end{equation*}
\end{theorem}

Theorem~\ref{thm:gaussian_noise2} is proved by the following lemmas.

\begin{lemma}
\label{lem:gaussian_noise_part1}
For $\zeta_1'(\varepsilon) = \sqrt{\frac{1}{2\ln\left(\frac{1}{\varepsilon}\right)}}$,
\begin{equation*}
\frac{r_{\Sigma,\varepsilon}(\x)}{\left\| \r^*_{2}(\x) \right\|_2} \geq \zeta_1'(\varepsilon) \frac{ \| \w \|_{2} }{\| \sqrt{\Sigma}\w \|_{2} }.
\end{equation*}
\end{lemma}
\begin{proof}
As in the proof of Lemma~\ref{lem:lp_noise_part1}:
\[
\mathbb{P}_{\vb\sim\mathcal{N}\left(\mathbf{0},\Sigma\right)}\left\{g\left(\x+\alpha\vb\right) \neq g\left(\x\right)\right\}=\mathbb{P}_{\vb\sim\mathcal{N}\left(\mathbf{0},\Sigma\right)}\left\{\left\Vert w\right\Vert _{2}\left\| \r^*_{2}(\x) \right\|_2\leq \left|\alpha\right|\w^{T}\vb\right\}\text{.}
\]
Since $\vb\sim\mathcal{N}\left(\mathbf{0},\Sigma\right)$ follows
a multivariate normal distribution with a positive definite covariance
matrix $\Sigma$, if $\sqrt{\Sigma}$ is the (symmetric) square root of
$\Sigma$, then $\vb=\sqrt{\Sigma}\vb'$ with $\vb'\sim\mathcal{N}\left(\mathbf{0},I_{d}\right)$.
So:
\begin{eqnarray*}
\mathbb{P}_{\vb\sim\mathcal{N}\left(\mathbf{0},\Sigma\right)}\left\{g\left(\x+\alpha\vb\right) \neq g\left(\x\right)\right\} & = & \mathbb{P}_{\vb\sim\mathcal{N}\left(\mathbf{0},I_{d}\right)}\left\{\left\Vert \w\right\Vert _{2}\left\| \r^*_{2}(\x) \right\|_2\leq\left|\alpha\right|\w^{T}\sqrt{\Sigma}\vb\right\}\\
 & = & \mathbb{P}_{\vb\sim\mathcal{N}\left(\mathbf{0},I_{d}\right)}\left\{\left\Vert \w\right\Vert _{2}\left\| \r^*_{2}(\x) \right\|_2\leq\left(\left|\alpha\right|\sqrt{\Sigma}\w\right)^{T}\vb\right\}\text{.}
\end{eqnarray*}
If $\vb\sim\mathcal{N}\left(\mathbf{0},I_{d}\right)$, then $\left(\left|\alpha\right|\sqrt{\Sigma}\w\right)^{T}\vb\sim\mathcal{N}\left(0,\alpha^2\Vert \sqrt{\Sigma}\w\Vert _{2}^{2}\right)$.
Therefore:
\[
\mathbb{P}_{\vb\sim\mathcal{N}\left(\mathbf{0},\Sigma\right)}\left\{g\left(\x+\alpha\vb\right) \neq g\left(\x\right)\right\}\leq\exp\left(-\frac{1}{2}\left(\frac{\left\Vert \w\right\Vert _{2}\left\| \r^*_{2}(\x) \right\|_2}{\alpha\Vert \sqrt{\Sigma}\w\Vert _{2}}\right)^{2}\right)\text{.}
\]
So, if $\left|\alpha\right|<\sqrt{\frac{1}{2\ln\frac{1}{\varepsilon}}}\frac{\left\Vert \w\right\Vert _{2}}{\Vert \sqrt{\Sigma}\w\Vert _{2}}\left\| \r^*_{2}(\x) \right\|_2$,
then $\mathbb{P}_{\vb\sim\mathcal{N}\left(\mathbf{0},\Sigma\right)}\left\{g\left(\x+\alpha\vb\right) \neq g\left(\x\right)\right\}<\varepsilon$.
Thus,
\begin{equation*}
\frac{r_{\Sigma,\varepsilon}(\x)}{\left\| \r^*_{2}(\x) \right\|_2} \geq \zeta_1'(\varepsilon) \frac{ \| \w \|_{2} }{\| \sqrt{\Sigma}\w \|_{2} }.
\end{equation*}
\end{proof}

\begin{lemma}
\label{lem:gaussian_noise_part2}
For $\varepsilon < \frac{1}{3}$ and $\zeta_2'(\varepsilon) = \sqrt{\frac{1}{1-\sqrt{3\varepsilon}}}$,
\begin{equation*}
\frac{r_{\Sigma,\varepsilon}(\x)}{\left\| \r^*_{2}(\x) \right\|_2} \leq \zeta_2'(\varepsilon) \frac{ \| \w \|_{2} }{\| \sqrt{\Sigma}\w \|_{2} }.
\end{equation*}
\end{lemma}
\begin{proof}
\begin{eqnarray*}
\mathbb{P}_{\vb\sim\mathcal{N}\left(\mathbf{0},\Sigma\right)}\left\{g\left(\x+\alpha\vb\right) \neq g\left(\x\right)\right\} & = & \mathbb{P}_{\vb\sim\mathcal{N}\left(\mathbf{0},I_{d}\right)}\left\{\left\Vert \w\right\Vert _{2}\left\| \r^*_{2}(\x) \right\|_2\leq\left(\left|\alpha\right|\sqrt{\Sigma}\w\right)^{T}\vb\right\}\\
 & = & \frac{1}{2}\mathbb{P}_{\vb\sim\mathcal{N}\left(\mathbf{0},I_{d}\right)}\left\{\left(\left\Vert \w\right\Vert _{2}\left\| \r^*_{2}(\x) \right\|_2\right)^{2}\leq\left(\left(\alpha\sqrt{\Sigma}\w\right)^{T}\vb\right)^{2}\right\}\\
 & = & \frac{1}{2}\mathbb{P}_{\vb\sim\mathcal{N}\left(\mathbf{0},I_{d}\right)}\left\{\left(\frac{1}{\alpha}\frac{\left\Vert \w\right\Vert _{2}}{\Vert \sqrt{\Sigma}\w\Vert _{2}}\left\| \r^*_{2}(\x) \right\|_2\right)^{2}\leq\frac{\left(\left(\sqrt{\Sigma}\w\right)^{T}\vb\right)^{2}}{\Vert \sqrt{\Sigma}\w\Vert _{2}^{2}}\right\}\text{.}
\end{eqnarray*}
Note that $\mathbb{E}_{\vb\sim\mathcal{N}\left(\mathbf{0},I_{d}\right)}\left(\frac{\left(\left(\sqrt{\Sigma}\w\right)^{T}\vb\right)^{2}}{\Vert \sqrt{\Sigma}\w\Vert _{2}^{2}}\right)=\frac{\var_{\vb\sim\mathcal{N}\left(\mathbf{0},I_{d}\right)}\left(\left(\sqrt{\Sigma}\w\right)^{T}\vb\right)}{\Vert \sqrt{\Sigma}\w\Vert _{2}^{2}}=1$.
So, by using Paley-Zygmund's inequality, when $\left|\alpha\right|\geq\frac{\left\Vert \w\right\Vert _{2}}{\Vert \sqrt{\Sigma}\w\Vert _{2}}\left\| \r^*_{2}(\x) \right\|_2$:
\begin{eqnarray*}
\mathbb{P}_{\vb\sim\mathcal{N}\left(\mathbf{0},\Sigma\right)}\left(g\left(\x+\alpha\vb\right) \neq g\left(\x\right)\right) & \geq & \frac{\left(1-\left(\frac{1}{\alpha}\frac{\left\Vert \w\right\Vert _{2}}{\Vert \sqrt{\Sigma}\w\Vert _{2}}\left\| \r^*_{2}(\x) \right\|_2\right)^{2}\right)^{2}}{\frac{2\Vert \sqrt{\Sigma}\w\Vert _{4}^{4}+\Vert \sqrt{\Sigma}\w\Vert _{2}^{4}}{\Vert \sqrt{\Sigma}\w\Vert _{2}^{4}}}=\frac{\left(1-\left(\frac{1}{\alpha}\frac{\left\Vert \w\right\Vert _{2}}{\Vert \sqrt{\Sigma}\w\Vert _{2}}\left\| \r^*_{2}(\x) \right\|_2\right)^{2}\right)^{2}}{2\left(\frac{\Vert \sqrt{\Sigma}\w\Vert _{4}}{\Vert \sqrt{\Sigma}\w\Vert _{2}}\right)^{4}+1}\\
 & \geq & \frac{\left(1-\left(\frac{1}{\alpha}\frac{\left\Vert \w\right\Vert _{2}}{\Vert \sqrt{\Sigma}\w\Vert _{2}}\left\| \r^*_{2}(\x) \right\|_2\right)^{2}\right)^{2}}{3}\text{.}
\end{eqnarray*}
So, if $\left|\alpha\right|>\frac{1}{\sqrt{1-\sqrt{3\varepsilon}}}\frac{\left\Vert \w\right\Vert _{2}}{\Vert \sqrt{\Sigma}\w\Vert _{2}}\left\| \r^*_{2}(\x) \right\|_2$,
then $\mathbb{P}_{\vb\sim\mathcal{N}\left(\mathbf{0},\Sigma\right)}\left\{g\left(\x+\alpha\vb\right) \neq g\left(\x\right)\right\}>\varepsilon$.
Therefore,
\begin{equation*}
\frac{r_{\Sigma,\varepsilon}(\x)}{\left\| \r^*_{2}(\x) \right\|_2} \leq \zeta_2'(\varepsilon) \frac{ \| \w \|_{2} }{\| \sqrt{\Sigma}\w \|_{2} }.
\end{equation*}
\end{proof}

\subsection{Typical Value of the Multiplicative Factor}

\begin{proposition}
\label{prop:MeanBoundsGaussian}
Let $\Sigma$ be a $d \times d$ positive semidefinite
matrix with $\tr\left(\Sigma\right) = 1$. If $\w$ is a random direction uniformly distributed over the unit
$\ell_{2}$-sphere, then, for $t \leq \frac{\sqrt{\pi}}{8}d$:
$$\mathbb{P}\left\{\left|\left(\frac{\left\Vert \w\right\Vert _{2}}{\Vert \sqrt{\Sigma}\w\Vert _{2}}\right)^2-d\right|\geq t'\right\} \leq 2\exp\left(-\frac{t^2}{8d}\right) + 2\exp\left(-\frac{t^{2}}{8d^2\tr\left(\Sigma^2\right)}\right) + 2\exp\left(-\frac{1}{200\tr\left(\Sigma^2\right)}\right),$$
where $t' = \frac{5}{2}t$.
\begin{equation*}
\end{equation*}
\end{proposition}
\begin{proof}
Suppose that $\w$ is a random direction uniformly distributed over the unit $\ell_{2}$ sphere.

Then
$\w$ can be written
as $\frac{\g}{\left\Vert \g\right\Vert _{2}}$, where $\g=\left(g_{1},\ldots,g_{d}\right)$
are i.i.d. with normal distribution ($\mu=0$, $\sigma^{2}=\frac{1}{2}$).
By using this representation in the orthogonal basis in which $\sqrt{\Sigma}$
is diagonal, we get
\[
\frac{\left\Vert \g\right\Vert _{2}}{\Vert \sqrt{\Sigma}\g\Vert _{2}}=\sqrt{\frac{\sum_{i=1}^{d}g_{i}^{2}}{\sum_{i=1}^{d}\left(\lambda_{i}g_{i}\right)^{2}}}\text{,}
\]
where $\sqrt{\Sigma}=\operatorname{Diag}\left(\left(\lambda_{i}\right)\right)$
in the previously mentioned orthogonal basis.

Let us focus on the concentration of $\sum_{i=1}^{d}\left(\lambda_{i}g_{i}\right)^{2}$.
We have:
\[
\sum_{i=1}^{d}\left(\lambda_{i}g_{i}\right)^{2}-\frac{1}{2} = \sum_{i=1}^{d}\left(\lambda_{i}g_{i}\right)^{2}-\frac{1}{2}\sum_{i=1}^{d}\lambda_{i}^{2}=\sum_{i=1}^{d}\lambda_{i}^{2}\left(g_{i}^{2}-\mathbb{E}\left(g_{i}^{2}\right)\right)\text{.}
\]
One of Bernstein-type inequalities \cite{bernstein1927probability} can be applied:
\[
\mathbb{P}\left\{\left|\sum_{i=1}^{d}\lambda_{i}^{2}\left(g_{i}^{2}-\mathbb{E}\left(g_{i}^{2}\right)\right)\right|\geq2t\sqrt{\var\left(g_{i}^{2}-\mathbb{E}\left(g_{i}^{2}\right)\right)\sum_{i=1}^{d}\lambda_{i}^{4}}\right\}\leq2e^{-t^{2}}\text{,}
\]
for $t \leq \beta \sqrt{\tr \left(\Sigma^2\right)}$ where $\beta = \frac{\sqrt{\pi}}{8}$ is a constant\footnote{Because $\frac{\Gamma\left(\frac{k+1}{2}\right)}{\sqrt{\pi}} = \mathbb{E}\left(\left|g_i\right|^k\right) \leq \frac{1}{2} \mathbb{E}\left(g_i^2\right) \left(\frac{4}{\sqrt{\pi}}\right)^{k-2} k!$ for all $k > 1$.}, i.e., for $t \leq \frac{\beta}{2}$:
\[
\mathbb{P}\left\{\left|\Vert \sqrt{\Sigma}\g\Vert _{2}-\frac{1}{2}\right|\geq t\right\}\leq2\exp\left(-\frac{t^{2}}{2\tr\left(\Sigma^2\right)}\right)\text{.}
\]

$2\Vert\g\Vert _{2}^2$ has a chi-squared distribution, so using a simple concentration inequality for the chi-squared distribution\footnote{Using the fact that $2\Vert\g\Vert _{2}^2$ is a sum of independent sub-exponential random variables (see \url{https://www.stat.berkeley.edu/~mjwain/stat210b/Chap2_TailBounds_Jan22_2015.pdf}, Example~2.5, for instance).}: $$\bb{P}\left\{\left|\frac{1}{d}\|\g\|_2^2-\frac{1}{2}\right| \geq t\right\} \leq 2\exp\left(-\frac{d t^2}{2}\right).$$

Overall, for $t \leq \beta d$ and $t' = \frac{5}{2}t$:
\begin{eqnarray*}
	\mathbb{P}\left\{\left|\left(\frac{\left\Vert \g\right\Vert _{2}}{\Vert \sqrt{\Sigma}\g\Vert _{2}}\right)^2-d\right|\geq t'\right\} & = & \mathbb{P}\left\{\left|\frac{\left\Vert \g\right\Vert _{2}^2 - d \Vert \sqrt{\Sigma}\g\Vert _{2}^2}{\Vert \sqrt{\Sigma}\g\Vert _{2} }\right|\geq t'\right\}\\
	& = & \mathbb{P}\left\{\left|\frac{\left(\left\Vert \g\right\Vert _{2}^2 - \frac{d}{2}\right) - d \left(\Vert \sqrt{\Sigma}\g\Vert _{2}^2 - \frac{1}{2}\right)}{\Vert \sqrt{\Sigma}\g\Vert _{2}^2}\right|\geq t'\right\}\\
	& \leq & \mathbb{P}\left\{\frac{\left|\frac{1}{d}\left\Vert \g\right\Vert _{2}^2 - \frac{1}{2}\right| + \left|\Vert \sqrt{\Sigma}\g\Vert _{2}^2 - \frac{1}{2}\right|}{\Vert \sqrt{\Sigma}\g\Vert _{2}^2}\geq \frac{t'}{d}\right\}\\
	& \leq & \mathbb{P}\left\{\left|\frac{1}{d}\left\Vert \g\right\Vert _{2}^2 - \frac{1}{2}\right| \geq \frac{t}{2d}\right\} + \mathbb{P}\left\{\left|\Vert \sqrt{\Sigma}\g\Vert _{2}^2 - \frac{1}{2}\right| \geq \frac{t}{2d}\right\}\\
	& & \ +\  \mathbb{P}\left\{\left|\Vert \sqrt{\Sigma}\g\Vert _{2}^2 - \frac{1}{2}\right| \geq \frac{1}{10}\right\},
\end{eqnarray*}
so, using the previous inequalities:$$\mathbb{P}\left\{\left|\left(\frac{\left\Vert \g\right\Vert _{2}}{\Vert \sqrt{\Sigma}\g\Vert _{2}}\right)^2-d\right|\geq t'\right\} \leq 2\exp\left(-\frac{t^2}{8d}\right) + 2\exp\left(-\frac{t^{2}}{8d^2\tr\left(\Sigma^2\right)}\right) + 2\exp\left(-\frac{1}{200\tr\left(\Sigma^2\right)}\right).$$

\end{proof}

\section{Robustness of LAF Classifiers to $\ell_p$ and Gaussian Noise}

\begin{theorem}
\label{thm:lp_noise_laf2}
Let $p \in [1, \infty]$. Let $p' \in [1,\infty]$ be such that $\frac{1}{p} + \frac{1}{p'} = 1$. Let $\varepsilon_0, \zeta_1(\varepsilon), \zeta_2(\varepsilon)$ be as in Theorem~\ref{thm:lp_noise2}. Then, for all  $\varepsilon < \varepsilon_0$, the following holds.

Assume $f$ is a classifier that is $(\gamma, \eta)$-LAF at point $\x$ and $\x^*$ be such that $\r^*_{p}(\x) = \x^* - \x$. Then:
\begin{equation*}
(1-\gamma) \zeta_1(\varepsilon) d^{1/p} \frac{ \| \nabla f(\x^*) \|_{p'} }{\| \nabla f(\x^*) \|_{2} } \leq  \frac{r_{p,\varepsilon}(\x)}{\| \r^*_{p}(\x) \|_p}
\end{equation*}
and
\begin{equation*}
\frac{r_{p,\varepsilon}(\x)}{\| \r^*_{p}(\x) \|_p} \leq (1+\gamma) \zeta_2(\varepsilon) d^{1/p} \frac{ \| \nabla f (\x^*) \|_{p'} }{\| \nabla f(\x^*) \|_{2} } ,
\end{equation*}
provided $$\eta \geq (1+\gamma) \zeta_2(\varepsilon) d^{1/p} \frac{ \| \nabla f (\x^*) \|_{p'} }{\| \nabla f(\x^*) \|_{2}} \left\|\r_p^*\left(\x\right) \right\|_p  = \eta_{\mathrm{lim}}.$$
\end{theorem}
\begin{proof}
Let $f_{-}$ and $f_{+}$ be functions such that the separating hyperplanes of, respectively, $\mathcal{H}_{\gamma}^{-}(\x,\x^*)$ and $\mathcal{H}_{\gamma}^{+}(\x,\x^*)$ are described
by equations, respectively, $f_{-}\left(\z\right)=0$ and $f_{+}\left(\z\right)=0$.
By definition, we know that $\left\|\r_p^*\left(f_-,\x\right) \right\|_p=\left(1-\gamma\right)\left\|\r_p^*\left(\x\right) \right\|_p$
and $\left\|\r_p^*\left(f_+,\x\right) \right\|_p=\left(1+\gamma\right)\left\|\r_p^*\left(\x\right) \right\|_p$.

From the definition
of LAF classifiers, since for all $\eta' \leq \frac{1-\gamma}{1+\gamma}\eta_{\mathrm{lim}}$,
$z\in\mathcal{H}_{\gamma}^-\left(\x,\x^*\right)\cap\mathcal{B}_p(\x, \eta')\Rightarrow f\left(\z\right)f\left(\x\right)>0$,
we have $r_{p,\varepsilon}(f_-,\x)\leq r_{p,\varepsilon}(\x)$;
indeed, if $\x+\alpha\vb$ with $\alpha\leq\frac{1-\gamma}{1+\gamma}\eta_{\mathrm{lim}}$ is not misclassified by $f_-$, then
it is not misclassified by $f$. Therefore, by applying Lemma~\ref{lem:lp_noise_part1}
to $f_{-}$, we get:
\begin{equation*}
\left(1-\gamma\right)\zeta_1\left(\varepsilon\right)d^{1/p}\frac{\left\Vert \nabla f(\x^{*})\right\Vert _{p'}}{\left\Vert \nabla f(\x^{*})\right\Vert _{2}} \leq \frac{r_{p,\varepsilon}(\x)}{\| \r^*_{p}(\x) \|_p}\text{.}
\end{equation*}
Since as long as $\eta'\leq\eta_{\mathrm{lim}}$, $z\in\mathcal{H}_{\gamma}^+\left(\x,\x^*\right)\cap\mathcal{B}_p(\x, \eta')\Rightarrow f\left(\z\right)f\left(\x\right)<0$,
we can apply a symmetric reasoning for $f_{+}$, and get:
\begin{equation*}
\frac{r_{p,\varepsilon}(\x)}{\| \r^*_{p}(\x) \|_p} \leq (1+\gamma) \zeta_2(\varepsilon) d^{1/p} \frac{ \| \nabla f (\x^*) \|_{p'} }{\| \nabla f(\x^*) \|_{2} } \text{.}
\end{equation*}
\end{proof}

\begin{theorem}
\label{thm:gaussian_noise_laf2}
Let $\Sigma$ be a $d \times d$ positive semidefinite matrix with $\tr(\Sigma) = 1$. Let $\varepsilon'_0, \zeta'_1(\varepsilon), \zeta'_2(\varepsilon)$ as in Theorem~\ref{thm:gaussian_noise2}. Then, for all $\varepsilon < \frac{1}{2}\varepsilon'_0$, the following holds.

Assume $f$ is a classifier that is $(\gamma, \eta)$-LAF at point $\x$ and $\x^*$ be such that $\r^*_{2}(\x) = \x^* - \x$. Then:
\begin{equation*}
(1-\gamma) \zeta'_1\left(\frac{\varepsilon}{2}\right) \frac{ \| \nabla f(\x^*) \|_{2} }{\| \sqrt{\Sigma} \nabla f(\x^*) \|_{2} } \leq  \frac{r_{\Sigma,\varepsilon}(\x)}{\| \r^*_{2}(\x) \|_2}
\end{equation*}
and
\begin{equation*}
\frac{r_{\Sigma,\varepsilon}(\x)}{\| \r^*_{2}(\x) \|_2} \leq (1+\gamma) \zeta'_2\left(\frac{3\varepsilon}{2}\right) \frac{ \| \nabla f (\x^*) \|_{2} }{\| \sqrt{\Sigma} \nabla f(\x^*) \|_{2} },
\end{equation*}
provided $$\eta \geq (1+\gamma) \left(1 + 8\tr\left(\Sigma^2\right) \ln \frac{4}{\varepsilon}\right) \zeta'_2\left(\frac{3\varepsilon}{2}\right) \frac{ \| \nabla f (\x^*) \|_{2} }{\| \sqrt{\Sigma} \nabla f(\x^*) \|_{2} }\left\|\r_2^*\left(\x\right) \right\|_2 = \eta_{\mathrm{lim}}.$$
\end{theorem}
\begin{proof}
This proof can be directly adapted from the proof of Theorem~\ref{thm:lp_noise_laf2}. The difference in the Gaussian case is that $\vb$ is no longer sampled from the unit ball, and its norm is not limited anymore. However, its norm can be bounded with high probability, and this enables to adapt the bounds of Theorem~\ref{thm:lp_noise_laf2} to the Gaussian case.

Indeed, using a Bernstein inequality as in the proof of Proposition~\ref{prop:MeanBoundsGaussian}, we have:
$$\mathbb{P}\left\{\left|\Vert \sqrt{\Sigma}\vb\Vert _{2}-1\right|\geq t\right\}\leq2\exp\left(-\frac{t^{2}}{8\tr\left(\Sigma^2\right)}\right) \leq \frac{\varepsilon}{2},$$
for $t = \psi(\varepsilon) = 8\tr\left(\Sigma^2\right) \ln \frac{4}{\varepsilon}$.

Let us focus on the upper bound for this proof; the lower bound follows by a similar reasoning. From the definition
of LAF classifiers, since for all for all $\eta' \leq \eta_{\mathrm{lim}}$,
$z\in\mathcal{H}_{\gamma}^+\left(\x,\x^*\right)\cap\mathcal{B}_p(\x, \eta')\Rightarrow f\left(\z\right)f\left(\x\right)<0$,
we have $r_{\Sigma,\varepsilon}(\x) \leq r_{\Sigma,\frac{3\varepsilon}{2}}(f_+,\x)$;
indeed, if $\x+\alpha\vb$ with $\alpha\leq\frac{\eta_{\mathrm{lim}}}{1+\psi(\varepsilon)}$ is misclassified by $f_+$, then
it is misclassified by $f$ if $\left\Vert \alpha\vb \right\Vert_2 \leq \eta_{\mathrm{lim}}$. Therefore, by applying Lemma~\ref{lem:gaussian_noise_part2}
to $f_{+}$:
$$\frac{r_{\Sigma,\varepsilon}(\x)}{\| \r^*_{2}(\x) \|_2} \leq (1+\gamma) \zeta'_2\left(\frac{\varepsilon}{2}\right) \frac{ \| \nabla f (\x^*) \|_{2} }{\| \sqrt{\Sigma} \nabla f(\x^*) \|_{2} }.$$
\end{proof}

\section{Generalization to Multi-class Classifiers}
We present in this section a generalization of Theorem~\ref{thm:lp_noise2} to multi-class linear classifiers, and discuss about the generalization of the other results to the multi-class case.

A classifier $f$ is said to be linear if for all $k \in \left\llbracket 1, L \right\rrbracket$, there are vector $\w_k,\bs{b}_k$ such that $f_k\left(\x\right) = \w_k^T \x + \bs{b}_k$. In this setting, Theorem~\ref{thm:lp_noise2} can be generalized by replacing $\zeta_1(\varepsilon)$ by $\zeta_1 \left(\frac{\varepsilon}{L-1}\right)$ in the lower bound.

\begin{theorem}
\label{thm:multiclass_lp_noise}
Let $p \in [1, \infty]$. Let $p' \in [1,\infty]$ be such that $\frac{1}{p} + \frac{1}{p'} = 1$. Let $\varepsilon_0, \zeta_1(\varepsilon), \zeta_2(\varepsilon)$ be the constants as defined in Theorem~\ref{thm:lp_noise2}. Let $k = g\left(\x\right)$ (the label attributed to $\x$ by $f$), $j$ be a class such that $\x + \r_{p}^*\left(\x\right)$ lies on the decision boundary between classes $k$ and $j$ (i.e., the class of the adversarial pertubation of $\x$) and $j' = \argmin_l{\frac{ \| \w_k - \w_l \|_{p'} }{\| \w_k - \w_l \|_{2} }}$. Then, for all $\varepsilon < \varepsilon_0$:
\begin{equation*}
\zeta_1 \left(\frac{\varepsilon}{L-1}\right) d^{1/p} \frac{ \| \w_k - \w_{j'} \|_{p'} }{\| \w_k - \w_{j'} \|_{2} } \leq \frac{r_{p,\varepsilon}(\x)}{\| \r^*_{p}(\x) \|_p} \leq \zeta_2(\varepsilon) d^{1/p} \frac{ \| \w_k - \w_j \|_{p'} }{\| \w_k - \w_j \|_{2} }.
\end{equation*}
\end{theorem}
\begin{proof}
We first define for the sake of the demonstration for any class $l$ the adversarial perturbation in the binary case where only classes $k$ and $l$ are considered:
\begin{equation*}
\r_p^*\left(\x, l\right) = \argmin_{\r} \| \r \|_p \text{ s.t. } f_k(\x+\r) < f_j(\x+\r).
\end{equation*}
It is then possible to express conveniently $\bb{P}_{\vb \sim \calB_p}\left\{g\left(\x\right) \neq g\left(\x+\alpha \vb\right)\right\}$:
\begin{eqnarray*}
\bb{P}_{\vb \sim \calB_p}\left\{g\left(\x\right) \neq g\left(\x+\alpha \vb\right)\right\} & = & \mathbb{P}_{\vb \sim \calB_p}\left\{\exists l\neq k, f_k(\x) < f_l(\x+\alpha\vb)\right\}\\
 & = & \mathbb{P}_{\vb \sim \calB_p}\left\{\exists l\neq k,\left(\w_l-\w_k\right)^{T}\vb\geq\frac{f_k\left(\x\right)-f_l\left(\x\right)}{\left|\alpha\right|}\right\}\\
 & = & \mathbb{P}_{\vb \sim \calB_p}\left\{\exists l\neq k,\frac{\left(\w_l-\w_k\right)^{T}}{\left\|\w_l-\w_k\right\|_{p'}}\vb\geq\frac{\r_p^*\left(\x, l\right)}{\left|\alpha\right|}\right\}\text{.}
\end{eqnarray*}

Let us first prove the inequality on the upper bound, as in Lemma~\ref{lem:lp_noise_part2}.
\begin{eqnarray*}
\bb{P}_{\vb \sim \calB_p}\left\{g\left(\x\right) \neq g\left(\x+\alpha \vb\right)\right\} & \geq & \mathbb{P}_{\vb \sim \calB_p}\left\{\frac{\left(\w_j-\w_k\right)^{T}}{\left\|\w_j-\w_k\right\|_{p'}}\vb\geq\frac{\r_p^*\left(\x, j\right)}{\left|\alpha\right|}\right\}\\
 & & = \mathbb{P}_{\vb \sim \calB_p}\left\{\frac{\left(\w_j-\w_k\right)^{T}}{\left\|\w_j-\w_k\right\|_{p'}}\vb\geq\frac{\r_p^*\left(\x\right)}{\left|\alpha\right|}\right\}\text{,}
\end{eqnarray*}
by definition of $j$. Then using the same reasoning as in Lemma~\ref{lem:lp_noise_part2} leads to
\begin{equation*}
\frac{r_{p,\varepsilon}(\x)}{\| \r^*_{p}(\x) \|_p} \leq \zeta_2(\varepsilon) d^{1/p} \frac{ \| \w_k - \w_j \|_{p'} }{\| \w_k - \w_j \|_{2} }.
\end{equation*}

Let us then prove the inequality on the lower bounds, as in Lemma~\ref{lem:lp_noise_part1}. We use the union bound to derive the inequality:
\begin{eqnarray*}
\bb{P}_{\vb \sim \calB_p}\left\{g\left(\x\right) \neq g\left(\x+\alpha \vb\right)\right\} & \leq & \sum_{l \neq k}\mathbb{P}_{\vb \sim \calB_p}\left\{\frac{\left(\w_l-\w_k\right)^{T}}{\left\|\w_l-\w_k\right\|_{p'}}\vb\geq\frac{\r_p^*\left(\x, l\right)}{\left|\alpha\right|}\right\}\\
 & \leq & \sum_{l \neq k} \mathbb{P}_{\vb \sim \calB_p}\left\{\frac{\left(\w_l-\w_k\right)^{T}}{\left\|\w_l-\w_k\right\|_{p'}}\vb\geq\frac{\r_p^*\left(\x\right)}{\left|\alpha\right|}\right\}\text{,}
\end{eqnarray*}
because $\r_p^*\left(\x\right) \geq \r_p^*\left(\x, l\right)$ for all $l$. Moreover, for $\left|\alpha\right| < \zeta_1\left(\frac{\varepsilon}{L-1}\right)d^{\frac{1}{p}}\frac{ \| \w_k - \w_{j'} \|_{p'} }{\| \w_k - \w_{j'} \|_{2} }\| \r_{p}^*(\x) \|_p$, by following the reasoning of Lemma~\ref{lem:lp_noise_part1} for each $l \neq k$:
\begin{equation*}
\bb{P}_{\vb \sim \calB_p}\left\{g\left(\x\right) \neq g\left(\x+\alpha \vb\right)\right\} \leq \sum_{l \neq k} \frac{\varepsilon}{L-1} = \varepsilon .
\end{equation*}
Therefore:
\begin{equation*}
\frac{r_{p,\varepsilon}(\x)}{\| \r^*_{p}(\x) \|_p} \geq \zeta_1 \left(\frac{\varepsilon}{L-1}\right) d^{1/p} \frac{ \| \w_k - \w_{j'} \|_{p'} }{\| \w_k - \w_{j'} \|_{2} }.
\end{equation*}
\end{proof}

The proof of this theorem uses the union bound to obtain the lower bound, explaining that $\zeta_1(\varepsilon)$ in the binary case becomes $\zeta_1 (\frac{\varepsilon}{L-1})$ in the multi-class setting. However, this inequality represents a worst case in the majoration used in the proof, and we observed in our experiments that using the coefficient $\zeta_1(\varepsilon)$ instead of $\zeta_1 (\frac{\varepsilon}{L-1})$ gives a proper lower bound on $\frac{r_{p,\varepsilon}(\x)}{\| \r^*_{p}(\x) \|_p}$.

Notice that it is possible to generalize other results that we proved in the binary case (Lemma~\ref{lem:lp_noise_alternative}, Theorems~\ref{thm:gaussian_noise2} and \ref{thm:lp_noise_laf2}) to the multi-class problem with a similar transormation of the inequalities (replacing $\zeta_1(\varepsilon)$ by $\zeta_1 (\frac{\varepsilon}{L-1})$ and using similar definitions of $j$ and $j'$).

\end{document}